\documentclass{article} 

\usepackage[final]{neurips_2021}

\usepackage[utf8]{inputenc} 
\usepackage[T1]{fontenc}    
\usepackage{hyperref}       
\usepackage{url}            
\usepackage{booktabs}       
\usepackage{amsfonts}       
\usepackage{nicefrac}       
\usepackage{microtype}      
\usepackage{xcolor}         
\usepackage{graphicx}
\usepackage{wrapfig}
\usepackage{xr} 

\usepackage{amsmath,amsfonts,bm}
\usepackage{multirow}
\usepackage{tabularx}
\usepackage{graphicx}
\usepackage{adjustbox}
\usepackage{amsthm}

\newcommand{\mset}[1]{\left\{\kern-.5em\left\{ #1 \right\}\kern-.5em\right\}}
\newcommand{\mmset}[1]{\{\kern-.4em\{ #1 \}\kern-.4em\}}

\newcommand{\vol}{V} 

\newcommand{\norm}[1]{\left\Vert#1\right\Vert}

\newcommand{\abs}[1]{\left\vert#1\right\vert}

\newcommand{\set}[1]{\left\{#1\right\}}
\newcommand{\parr}[1]{\left (#1\right )}
\newcommand{\brac}[1]{\left [#1\right ]}
\newcommand{\ip}[1]{\left \langle #1 \right \rangle }
\newcommand{\Real}{\mathbb R}

\newcommand{\eps}{\varepsilon}
\newcommand{\too}{\rightarrow}

\newcommand{\eg}{{e.g.}}
\newcommand{\ie}{{i.e.}}

 \newtheorem{theorem}{Theorem}
 \newtheorem{lemma}{Lemma}
 
 \newtheorem{corollary}{Corollary}

\def\eqref#1{equation~\ref{#1}}

\def\1{\bm{1}}

\def\eps{{\epsilon}}

\def\ve{{\bm{e}}}

\def\vn{{\bm{n}}}

\def\vu{{\bm{u}}}
\def\vv{{\bm{v}}}

\def\vx{{\bm{x}}}

\def\vy{{\bm{y}}}
\def\vz{{\bm{z}}}
\def\vec1{{\bm{1}}}

\def\mI{{\bm{I}}}

\def\mL{{\bm{L}}}
\def\mM{{\bm{M}}}
\def\mN{{\bm{N}}}

\def\mP{{\bm{P}}}

\DeclareMathAlphabet{\mathsfit}{\encodingdefault}{\sfdefault}{m}{sl}
\SetMathAlphabet{\mathsfit}{bold}{\encodingdefault}{\sfdefault}{bx}{n}

\def\gA{{\mathcal{A}}}

\def\gM{{\mathcal{M}}}

\def\gX{{\mathcal{X}}}
\def\gY{{\mathcal{Y}}}

\newcommand{\E}{\mathbb{E}}

\newcommand{\R}{\mathbb{R}}

\title{Moser Flow: Divergence-based Generative Modeling on Manifolds}

\author{%
  Noam Rozen$^1$ \And 
  Aditya Grover$^{2,3}$ \And
  Maximilian Nickel$^2$ \And
  Yaron Lipman$^{1,2}$ \AND
  \vspace{-25pt} \\
  $^1$Weizmann Institute of Science \qquad
  $^2$Facebook AI Research \quad
  $^3$UCLA
}

\begin{document}
\maketitle

\begin{abstract}
We are interested in learning generative models for complex geometries described via manifolds, such as spheres, tori, and other implicit surfaces. 
Current extensions of existing (Euclidean) generative models are restricted to specific geometries and typically suffer from high computational costs. We introduce \textit{Moser Flow} (MF), a new class of generative models within the family of continuous normalizing flows (CNF). MF also produces a CNF via a solution to the change-of-variable formula, however differently from other CNF methods, its model (learned) density is parameterized as the source (prior) density minus the \textit{divergence} of a neural network (NN). The divergence is a local, linear differential operator, easy to approximate and calculate on manifolds. Therefore, 
unlike other CNFs, MF does not require invoking or backpropagating through an ODE solver during training. Furthermore, representing the model density explicitly as the divergence of a NN rather than as a solution of an ODE facilitates learning high fidelity densities. Theoretically, we prove that MF constitutes a universal density approximator under suitable assumptions. Empirically, we demonstrate for the first time the use of flow models for sampling from general curved surfaces and achieve significant improvements in density estimation, sample quality, and training complexity over existing CNFs on challenging synthetic geometries and real-world benchmarks from the earth and climate sciences. 


\end{abstract}

\section{Introduction}
\label{s:intro}

The major successes of deep generative models in recent years are primarily in domains involving Euclidean data, such as images~\citep{dhariwal2021diffusion}, text~\citep{brown2020language}, and video~\citep{kumar2019videoflow}.
However, many kinds of scientific data in the real world lie in non-Euclidean spaces specified as manifolds. 
Examples include planetary-scale data for earth and climate sciences~\citep{mathieu2020riemannian}, protein interactions and brain imaging data for life sciences~\citep{gerber2010manifold,chen2012triangulated}, as well as 3D shapes in computer graphics~\citep{hoppe1992surface,kazhdan2006poisson}. Existing (Euclidean) generative models cannot be effectively applied in these scenarios as they would tend to assign some probability mass to areas outside the natural geometry of these domains.

An effective way to impose geometric domain constraints for deep generative modeling is to design \textit{normalizing flows} that operate in the desired manifold space.
A normalizing flow maps a prior (source) distribution to a target distribution via the change of variables formula~\citep{rezende2015variational,dinh2016density,papamakarios2019normalizing}.
Early work in this direction proposed invertible architectures for learning probability distributions directly over the specific manifolds defined over spheres and tori~\citep{rezende2020normalizing}.
Recently, ~\cite{mathieu2020riemannian} proposed to extend continuous normalizing flows (CNF)~\citep{chen2018neural} for generative modeling over Riemannian manifolds wherein the flows are defined via vector fields on manifolds and computed as the solution to an associated ordinary differential equation (ODE).
CNFs have the advantage that the neural network architectures parameterizing the flow need not be restricted via invertibility constraints.
However, as we show in our experiments, existing CNFs such as FFJORD~\citep{grathwohl2018ffjord} and Riemannian CNFs~\citep{mathieu2020riemannian} can be slow to converge and the generated samples can be inferior in capturing the details of high fidelity data densities.
Moreover, it is a real challenge to apply Riemannian CNFs to complex geometries such as general curved surfaces.


To address these challenges, we propose Moser Flows (MF), a new class of deep generative models within the CNF family.
An MF models the desired target density as the source density minus the divergence of an (unrestricted) neural network.
The divergence is a local, linear differential operator, easy to approximate and calculate on manifolds.
By drawing on classic results in differential geometry by \citet{moser1965volume} and \citet{dacorogna1990partial}, we can show that this parameterization induces a CNF solution to the change-of-variables formula specified via an ODE.
Since MFs directly parameterize the model density using the divergence, unlike other CNF methods, we do not require to explicitly solve the ODE for maximum likelihood training.
At test-time, we use the ODE solver for generation.
We derive extensions to MFs for Euclidean submanifolds that efficiently parameterize vector fields projected to the desired manifold domain.
Theoretically, we prove that Moser Flows are universal generative models over Euclidean submanifolds. That is, given a Euclidean submanifold $\gM$ and a target continuous positive probability density $\mu$ over $\gM$, MFs can push arbitrary positive source density $\nu$ over $\gM$ to densities $\bar{\mu}$ that are arbitrarily close to $\mu$.

We evaluate Moser Flows on a wide range of challenging real and synthetic problems defined over many different domains.
On synthetic problems, we demonstrate improvements in convergence speed for attaining a desired level of details in generation quality.
We then experiment with two kinds of complex geometries.
First, we show significant improvements of $49\%$ on average over Riemannian CNFs~\citep{mathieu2020riemannian} for density estimation as well as high-fidelity generation on 4 earth and climate science datasets corresponding to global locations of volcano eruptions, earthquakes, floods, and wildfires on spherical geometries.
Next and last, we go beyond spherical geometries to demonstrate for the first time, generative models on general curved surfaces.

\section{Preliminaries}
\label{s:prelims}


%

\textbf{Riemannian manifolds.} We consider an orientable, compact, boundaryless, connected $n$-dimensional \emph{Riemannian manifold} $\gM$ with metric $g$. We denote points in the manifold by $x,y\in\gM$. At every point $x\in\gM$, $T_x\gM$ is an $n$-dimensional tangent plane to $\gM$. The metric $g$ prescribes an inner product on each tangent space; for $v,u\in T_x\gM$, their inner product w.r.t.~$g$ is denoted by $\ip{v,u}_g$. $\mathfrak{X}(\gM)$ is the space of smooth (tangent) vector fields to $\gM$; that is, if $u\in\mathfrak{X}(\gM)$ then $u(x)\in T_x\gM$, for all $x\in\gM$, and if $u$ written in local coordinates it consists of smooth functions. 
We denote by $d\vol$ the \emph{Riemannian volume form}, defined by the metric $g$ over the manifold $\gM$. In particular, $\vol(\gA)=\int_{\gA}d\vol$ is the volume of the set $\gA\subset\gM$.

We consider \emph{probability measures} over $\gM$ that are represented by strictly positive continuous density functions $\mu,\nu:\gM\too\Real_{>0}$, where $\mu$ by convention represents the target (unknown) distribution and $\nu$ represents the source (prior) distribution. $\mu,\nu$ are probability densities in the sense their integral w.r.t.~the Riemannian volume form is one, \ie, $\int_\gM \mu d\vol=1=\int_\gM \nu d\vol$.
It is convenient to consider the volume forms that correspond to the probability measures, namely $\hat{\mu}=\mu d\vol$ and $\hat{\nu} = \nu d\vol$. Volume forms are differential $n$-forms that can be integrated over subdomains of $\gM$, for example, $p_\nu(\gA)=\int_\gA \hat{\nu}$ is the probability of the event $\gA\subset\gM$. 



\textbf{Continuous Normalizing Flows (CNF) on manifolds}~operate by transforming a simple source distribution through a map $\Phi$ into a highly complex and multimodal target distribution. A manifold CNF, $\Phi:\gM\too\gM$, is an orientation preserving diffeomorphism from the manifold to itself \citep{mathieu2020riemannian,lou2020neural,falorsi2020Neural}.
A smooth map $\Phi:\gM\too\gM$ can be used to \emph{pull-back} the target $\hat{\mu}$ according to the formula:
\begin{equation}\label{e:pull_back}
(\Phi^*\hat{\mu})_z (v_1,\ldots,v_n) = \hat{\mu}_{\Phi(z)}(D\Phi_z(v_1),\ldots,D\Phi_z(v_n)),    
\end{equation}
where $v_1,\ldots,v_n\in T_z \gM$ are arbitrary tangent vectors,  $D\Phi_z:T_z \gM \too T_{\Phi(z)} \gM$ is the differential of $\Phi$, namely a linear map between the tangent spaces to $\gM$ at the points $z$ and $\Phi(z)$,  respectively.
By pulling-back $\hat{\mu}$ according to $\Phi$ and asking it to equal to the prior density $\nu$, we get the manifold version of the standard \emph{normalizing equation}:
\begin{equation}\label{e:normalizing}
\hat{\nu} = \Phi^* \hat{\mu}.
\end{equation}
If the normalizing equation holds, then for an event $\gA\subset\gM$ we have that 
$$p_\nu(\gA)=\int_\gA \hat{\nu} = \int_{\gA} \Phi^*\hat{\mu}=\int_{\Phi(\gA)}\hat{\mu}=p_\mu(\Phi(\gA)).$$
Therefore, given a random variable $\vz$ distributed according to $\nu$, then $\vx=\Phi(\vz)$ is distributed according to $\mu$, and $\Phi$ is the \emph{generator}.

One way to construct a CNF $\Phi$ is by solving an ordinary differential equation (ODE) \citep{chen2018neural,mathieu2020riemannian}. Given a time-dependent vector field $v_t \in \mathfrak{X}(\gM)$ with $t\in [0,1]$, a one-parameter family of diffeomorphisms (CNFs) $\Phi_t:[0,1]\times \gM\too\gM$ is defined by 
\begin{equation}\label{e:Phi_ode}
\frac{d}{dt}\Phi_t  = v_t(\Phi_t),
\end{equation}
where this ODE is initialized with the identity transformation, \ie, for all $x\in\gM$ we initialize $\Phi_0(x)=x$. The CNF is then defined by $\Phi(x)=\Phi_1(x)$.

\textbf{Example: Euclidean CNF.} Let us show how the above notions boil down to standard Euclidean CNF for the choice of $\gM=\Real^n$, and the standard Euclidean metric; we denote $\vz=(z^1,\ldots,z^n)\in\Real^n$. The Riemannian volume form in this case is $d\vz=dz^1\wedge dz^2
\wedge \cdots \wedge dz^n$. Furthermore,  $\hat{\mu}(\vz)=\mu(\vz) d\vz$ and $\hat{\nu}(\vz)=\nu(\vz) d\vz$. The pull-back formula (\eqref{e:pull_back}) in coordinates (see \eg, Proposition 14.20 in \cite{lee2013smooth}) is $$\Phi^*\hat{\mu}(\vz) = \mu(\Phi(\vz))\det (D\Phi_\vz) d\vz,$$
where $D\Phi_\vz$ is the matrix of partials of $\Phi$ at point $\vz$, $(D\Phi_\vz)_{ij} = \frac{\partial\Phi^i}{\partial z^j}(\vz)$.
Plugging this in \eqref{e:normalizing} we get the Euclidean normalizing equation:
\begin{equation}
    \nu(\vz) = \mu(\Phi(\vz)) \det(D\Phi_\vz).
\end{equation}


\section{Moser Flow}

\cite{moser1965volume} and \cite{dacorogna1990partial} suggested a method for solving the normalizing equation, that is \eqref{e:normalizing}. Their method explicitly constructs a vector field $v_t$, and the flow it defines via \eqref{e:Phi_ode} is guaranteed to solve \eqref{e:normalizing}. 
We start by introducing the method, adapted to our needs, followed by its application to generative modeling. We will use notations introduced above. 

\subsection{Solving the normalizing equation}
Moser's approach to solving \eqref{e:normalizing} starts by interpolating the source and target distributions. That is, choosing an interpolant $\alpha_t:[0,1]\times\gM\too \Real_{>0}$, such that $\alpha_0=\nu$,  $\alpha_1=\mu$, and $\int_\gM\alpha_t dV=1$ for all $t\in[0,1]$. Then, a time-dependent vector field $v_t\in\mathfrak{X}(\gM)$ is defined so that for each time $t\in[0,1]$, the flow $\Phi_t$ defined by \eqref{e:Phi_ode} satisfies the \emph{continuous normalization equation}:
\begin{equation}\label{e:cont_normalizing}
    \Phi_t^*\hat{\alpha}_t = \hat{\alpha}_0,
\end{equation}
where $\hat{\alpha}_t=\alpha_t d\vol$ is the volume form corresponding to the density $\alpha_t$. Clearly, plugging $t=1$ in the above equation provides a solution to \eqref{e:normalizing} with $\Phi=\Phi_1$.
As it turns out, considering the continuous normalization equation simplifies matters and the sought after vector field $v_t$ is constructed as follows. First, solve the partial differential equation (PDE) over the manifold $\gM$
\begin{equation}\label{e:pde_general}
    \mathrm{div}(u_t) = -\frac{d}{dt}\alpha_t,
\end{equation}
where $u_t\in\mathfrak{X}(\gM)$ is an unknown time-dependent vector field, and $\mathrm{div}$ is the Riemannian generalization to the Euclidean divergence operator, $\mathrm{div}_E = \nabla \cdot$. This manifold divergence operator is defined by replacing the directional derivative of the Euclidean space with its Riemannian version, namely the covariant derivative, 
\begin{equation}\label{def:div}
    \mathrm{div} (u) = \sum_{i=1}^n \ip{\nabla_{e_i}u,e_i}_g,
\end{equation}
where $\set{e_i}_{i=1}^n$ is an orthonormal frame according to the Riemannian metric $g$, and $\nabla_\xi u$ is the Riemannian covariant derivative. 
Note that here we assume that $\gM$ is boundaryless, otherwise we need $u_t$ to be also tangent to the boundary of $\gM$. 
Second, define 
\begin{equation}\label{e:v_t_general}
    v_t=\frac{u_t}{\alpha_t}.
\end{equation}
Theorem 2 in \cite{moser1965volume} implies:
\begin{theorem}[Moser]\label{thm:moser}
The diffeomorphism $\Phi=\Phi_1$, defined by the ODE in \eqref{e:Phi_ode} and vector field $v_t$ in \eqref{e:v_t_general} solves the normalization equation, \ie, \eqref{e:normalizing}.
\end{theorem}
\begin{wrapfigure}[20]{r}{0.37\textwidth}
\vspace{-30pt}
  \begin{center}
    \includegraphics[width=0.35\textwidth]{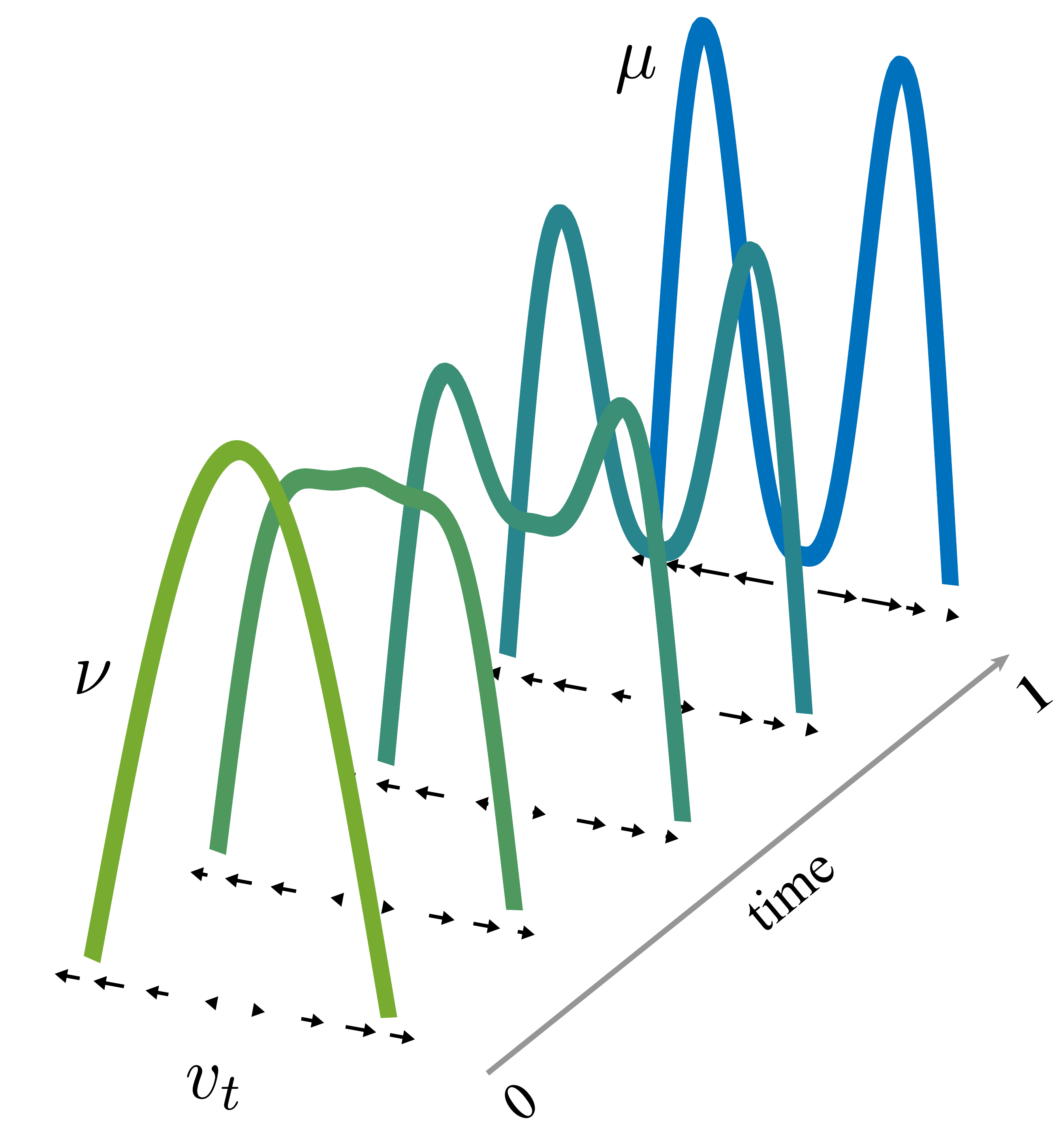}
  \end{center}
  \caption{1D example of Moser Flow: source density $\nu$ in green, target $\mu$ in blue. The vector field $v_t$ (black) is guaranteed to push $\nu$ to interpolated density $\alpha_t$ at time $t$, \ie, $(1-t)\nu+t\mu$.}\label{fig:1d}
\end{wrapfigure}
The proof of this theorem in our case is provided in the supplementary for completeness. A simple choice for the interpolant $\alpha_t$ that we use in this paper was suggested in  \cite{dacorogna1990partial}: 
\begin{equation}\label{e:interpolant}
    \alpha_t=(1-t)\nu + t\mu.
\end{equation}
The time derivative of this interpolant, \ie, $\frac{d}{dt}\alpha_t=\mu-\nu$, does not depend on $t$. Therefore the vector field can be chosen to be constant over time, $u_t=u$, and the PDE in \eqref{e:pde_general} takes the form \begin{equation}\label{e:pde}
    \mathrm{div}(u) = \nu-\mu,
\end{equation}
and consequently $v_t$ takes the form
\begin{equation}\label{e:v_t}
    v_t=\frac{u}{(1-t)\nu+t\mu}.
\end{equation}
Figure \ref{fig:1d} shows a one dimensional illustration of Moser Flow. 

\subsection{Generative model utilizing Moser Flow}
We next utilize MF to define our generative model. 
%
Our model (learned) density $\bar{\mu}$ is motivated from \eqref{e:pde} and is defined by
\begin{equation}\label{e:model}
    \bar{\mu} = \nu - \mathrm{div}(u),
\end{equation}
where $u\in \mathfrak{X}(\gM)$ is the degree of freedom of the model. We model this degree of freedom, $u$, with a deep neural network, more specifically, Multi-Layer Perceptron (MLP). We denote $\theta\in\Real^p$ the learnable parameters of $u$. We start by noting that, by construction, $\bar{\mu}$ has a unit integral over $\gM$:
\begin{lemma}\label{lem:int_mu}
If $\gM$ has no boundary, or $\vu\vert_{\partial \gM}\in \mathfrak{X}(\partial\gM)$, then $\int_\gM \bar{\mu} d\vol = 1$.
\end{lemma}
This lemma is proved in the supplementary and a direct consequence of Stokes' Theorem. If $\bar{\mu}>0$ over $\gM$ then, together with the fact that $\int_\gM \bar{\mu} d\vol=1$ (Lemma \ref{lem:int_mu}), it is a probability density over $\gM$. Consequently, Theorem \ref{thm:moser} implies that $\bar{\mu}$ is realized by a CNF defined via $v_t$:
\begin{corollary}\label{thm:CNF}
If $\bar{\mu}>0$ over $\gM$ then $\bar{\mu}$ is a probability distribution over $\gM$, and is generated by the flow $\Phi=\Phi_1$, where $\Phi_t$ is the solution to the ODE in \eqref{e:Phi_ode} with the vector field $v_t\in\mathfrak{X}(\gM)$ defined in \eqref{e:v_t}.
\end{corollary}
Since $\bar{\mu}>0$ is an open constraint and is not directly implementable, we replace it with the closed constraint $\bar{\mu}\geq \eps$, where $\eps>0$ is a small hyper-parameter. We define 
\begin{align*}
    \bar{\mu}_+(x) = \max\set{\eps, \bar{\mu}(x)}; \quad \bar{\mu}_-(x) = \eps - \min\set{\eps,\bar{\mu}(x)}.
\end{align*}
As can be readily verified:
\begin{equation}\label{e:bar_mu}
 \bar{\mu}_+,\bar{\mu}_- \geq 0,\ \text{ and  } \bar{\mu} = \bar{\mu}_+ - \bar{\mu}_-.
\end{equation}

We are ready to formulate the loss for training the generative model. Consider an unknown target distribution $\mu$, provided to us as a set of i.i.d.~observations $\gX=\set{x_i}_{i=1}^m \subset \gM$. Our goal is to maximize the likelihood of the data $\gX$ while making sure $\bar{\mu}\geq \epsilon$. We therefore consider the following loss:
\begin{equation}\label{e:loss}
\ell(\theta) = -\E_{\mu} \log \bar{\mu}_+(x) + \lambda \int_\gM  \bar{\mu}_-(x) d\vol
\end{equation}
where $\lambda$ is a hyper-parameter. The first term in the loss is approximated by the empirical mean computed with the observations $\gX$, \ie, $$\E_{x\sim \mu} \log \bar{\mu}_+(x) \approx \frac{1}{m}\sum_{i=1}^m \log \bar{\mu}_+(x_i).$$ This term is merely the negative log likelihood of the observations. 

The second term in the loss penalizes the deviation of $\bar{\mu}$ from satisfying $\bar{\mu}\geq \eps$. According to Theorem \ref{thm:CNF}, this measures the deviation of $\bar{\mu}$ from being a density function and realizing a CNF. One point that needs to be verified is that the combination of these two terms does not push the minimum away from the target density $\mu$. This can be verified with the help of the generalized Kullback–Leibler (KL) divergence providing a distance measure between arbitrary positive functions $f,g:\gM\too\Real_{>0}$:
\begin{equation}\label{e:kl}
    D(f,g) = \int_\gM f \log\parr{\frac{f}{g}}d\vol - \int_\gM fd\vol + \int_\gM g d\vol.
\end{equation}
Using the generalized KL, we can now compute the distance between the positive part of our model density, \ie, $\bar{\mu}_+$, and the target density:
\begin{align*}
D(\mu,\bar{\mu}_+) &= \E_\mu \log\parr{\frac{\mu}{\bar{\mu}_+}} - \int_\gM \mu d\vol + \int_\gM \bar{\mu}_+ d\vol   \\
&= \E_\mu \log \mu - \E_\mu \log \bar{\mu}_+  + \int_\gM \bar{\mu}_- d\vol
\end{align*}
where in the second equality we used Lemma \ref{lem:int_mu}. The term $\E_\mu \log \mu$ is the negative entropy of the unknown target distribution $\mu$. The loss in \eqref{e:loss} equals $D(\mu,\bar{\mu}_+)-\E_\mu \log \mu + (\lambda-1)\int_\gM \bar{\mu}_-d\vol$. 
Therefore, if $\lambda \geq 1$, and $\min_{x\in\gM} \mu(x) > \eps$ (we use the compactness of $\gM$ to infer existence of such a minimal positive value), then the unique minimum of the loss in \eqref{e:loss} is the target density, \ie, $\bar{\mu}=\mu$. Indeed, the minimal value of this loss is $-\E_\mu \log \mu$ and it is achieved by setting $\bar{\mu}=\mu$. Uniqueness follows by considering an arbitrary minimizer $\bar{\mu}$. Since such a minimizer satisfies $D(\mu,\bar{\mu}_+)=0$ and $\int_\gM \bar{\mu}_- d\vol=0$, necessarily $\bar{\mu}=\mu$. We proved:
\begin{theorem}\label{thm:min}
For $\lambda\geq 1$ and sufficiently small $\eps>0$, the unique minimizer of the loss in \eqref{e:loss} is $\bar{\mu}=\mu$.
\end{theorem}

\paragraph{Variation of the loss.} Lemma \ref{lem:int_mu} and \eqref{e:bar_mu} imply that $\int_\gM \bar{\mu}_+ d\vol=\int_\gM \bar{\mu}_- d\vol + 1$. Therefore, an equivalent loss to the one presented in \eqref{e:loss} is:
\begin{equation}\label{e:loss_variation}
    \ell(\theta) = -\E_{\mu} \log \bar{\mu}_+(x) + \lambda_- \int_\gM  \bar{\mu}_- d\vol + \lambda_+ \int_\gM  \bar{\mu}_+ d\vol
\end{equation}
with $\lambda_- + \lambda_+ \geq 1$.
Empirically we found this loss favorable in some cases (\ie, with $\lambda_+>0$). 

\paragraph{Integral approximation.}
The integral $\int_\gM \bar{\mu}_-d\vol$ in the losses in \eqref{e:loss_variation} and \ref{e:loss} is approximated by considering a set $\gY=\set{y_j}_{j=1}^l$ of i.i.d.~samples according to some distribution $\eta$ over $\gM$ and taking a Monte Carlo estimate
\begin{equation}\label{e:mc}
    \int_\gM  \bar{\mu}_- d\vol  \approx \frac{1}{l}\sum_{j=1}^l \frac{\bar{\mu}_-(y_j)}{\eta(y_j)}.
\end{equation}
$\int_\gM \bar{\mu}_+ d\vol$ is approximated similarly. 
In this paper we opted for the simple choice of taking $\eta$ to be the uniform distribution over $\gM$.



\section{Generative modeling over Euclidean submanifolds}
In this section, we adapt the Moser Flow (MF) generative model to submanifolds of Euclidean spaces. That is we consider an orientable, compact, boundaryless, connected $n$-dimensional submanifold $\gM\subset \Real^d$, where $n < d$. 
Examples include implicit surfaces and manifolds (\ie, preimage of a regular value of a smooth function), as well as triangulated surfaces and manifold simplicial complexes. 
%
We denote points in $\Real^d$ (and therefore in $\gM$) with $\vx,\vy\in\Real^d$. As the Riemannian metric of $\gM$ we take the induced metric from $\Real^d$; that is given arbitrary tangent vectors $\vv,\vu\in T_\vx\gM$ the metric is defined by $\ip{\vv,\vu}_g=\ip{\vv,\vu}$, where the latter is the Euclidean inner product. We denote by $\pi:\Real^d\too\gM$ the closest point projection on $\gM$, \ie, 
$\pi(\vx) = \min_{\vy\in\gM}\norm{\vx-\vy}$, with $\norm{\vy}^2=\ip{\vy,\vy}$ the Euclidean norm in $\Real^d$. 
Lastly, we denote by $\mP_\vx\in\Real^{d\times d}$ the orthogonal projection matrix on the tangent space $T_\vx\gM$; in practice if we denote by $\mN\in\Real^{d\times k}$ the matrix with orthonormal columns spanning $N_\vx \gM=(T_x \gM)^\perp$ (\ie, the normal space to $\gM$ at $\vx$) then, $\mP_\vx = \mI - \mN\mN^T$.

We parametrize the vector field $\vu$ required for our MF model (in \eqref{e:model}) by defining a vector field $\vu\in\mathfrak{X}(\Real^d)$ such that  $\vu\vert_{\gM}\in \mathfrak{X}(\gM)$. We define  
\begin{equation}\label{e:vu}
    \vu(\vx) = \mP_{\pi(\vx)} \vv_\theta(\pi(\vx)),
\end{equation}
where $\vv_\theta :\Real^d\too\Real^d$ is an MLP with Softplus activation ($\beta=100$) and learnable parameters $\theta\in\Real^p$. By construction, for $\vx\in \gM$, $\vu(\vx)\in T_\vx\gM$. 

To realize the generative model, we need to compute the divergence $\mathrm{div} (\vu(\vx))$ for $\vx\in\gM$ with respect to the Riemannian manifold $\gM$ and metric $g$. The vector field $\vu$ in \eqref{e:vu} is constant along normal directions to the manifold at $\vx$ (since $\pi(\vx)$ is constant in normal directions). If $\vn \in N_\vx\gM$, then in particular
\begin{equation}\label{e:n_constant}
    \frac{d}{dt}\Big\vert_{t=0}\vu(\vx+t\vn) = 0.
\end{equation}
We call vector fields that satisfy \eqref{e:n_constant} infinitesimally constant in the normal direction. As we show next, such vector fields $\vu\in \mathfrak{X}(\gM)$ have the useful property that their divergence along the manifold $\gM$ coincides with their Euclidean divergence in the ambient space $\Real^d$:
\begin{lemma}\label{lem:div}
If $\vu\in\mathfrak{X}(\Real^d)$, $\vu\vert_\gM\in\mathfrak{X}(\gM)$ is infinitesimally constant in normal directions of $\gM$, then for $\vx\in\gM$, $\mathrm{div}(\vu(\vx))=\mathrm{div}_E(\vu(\vx))$, where $\mathrm{div}_E$ denotes the standard Euclidean divergence. 
\end{lemma}
This lemma simply means we can compute the Euclidean divergence of $\vu$ in our implementation. 
Given a set of observed data $\gX=\set{\vx_i}_{i=1}^m \subset \gM \subset \Real^d$, and a set of uniform i.i.d.~samples $\gY=\set{\vy_j}_{j=1}^l \subset \gM$ over $\gM$, our loss (\eqref{e:loss}) takes the form
\begin{align*}
    \ell(\theta) &= -\frac{1}{m}\sum_{i=1}^m \log \max\set{\eps,\nu(\vx_i)-\mathrm{div}_E\vu(\vx_i)} + \frac{\lambda'_-}{l}\sum_{j=1}^l \Big( \eps - \min\set{\eps,\nu(\vy_j)-\mathrm{div}_E\vu(\vy_j)} \Big ),
\end{align*}
where $\lambda_-' = \lambda_- \vol(M)$. We note the volume constant can be ignored by considering an un-normalized source density $\vol(M)\nu \equiv 1$, see supplementary for details.
The loss in \eqref{e:loss_variation} is implemented similarly, namely, we add the empirical approximation of $\lambda_+ \int_\gM \bar{\mu}_+ d\vol$.

We conclude this section by stating that the MF generative model over Euclidean submanifolds (defined with equations \ref{e:model} and \ref{e:vu}) is universal. That is, MFs can generate, arbitrarily well, any continuous target density $\mu$ on a submanifold manifold $\gM \subset \Real^d$. 
\begin{theorem}\label{thm:universality}
Given an orientable, compact, boundaryless, connected, differentiable $n$-dimensional submanifold $\gM\subset\Real^d$, $n< d$, and a target continuous probability density $\mu:\gM\too\Real_{>0}$, there exists for each $\eps>0$ an MLP $\vv_\theta$ and a choice of weights $\theta$ so that $\bar{\mu}$ defined by equations \ref{e:model} and \ref{e:vu} satisfies $$\max_{\vx\in\gM} \abs{\mu(\vx) - \bar{\mu}(\vx)} < \eps.$$
\end{theorem}

\begin{figure*}
\centering
\begin{tabular}{@{\hskip2.5pt}c@{\hskip2.5pt}c@{\hskip2.5pt}c@{\hskip2.5pt}c@{\hskip2.5pt}c@{\hskip2.5pt}c@{\hskip2.5pt}}
    input data & samples & density & \hspace{0.15cm} input data & samples & density \\
    
    \includegraphics[width=20mm]{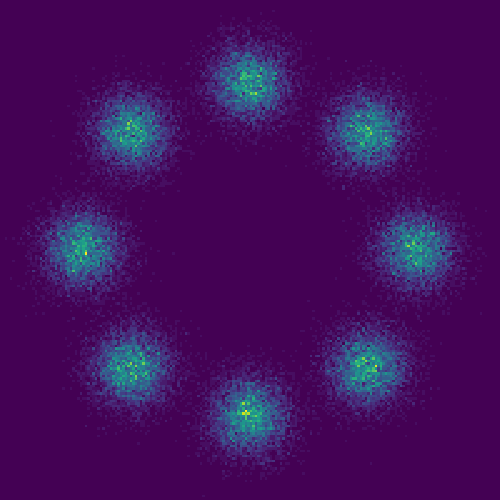}
    &
    \includegraphics[width=20mm]{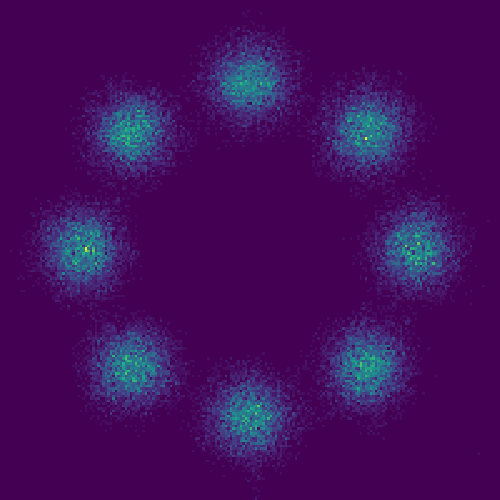}
    &
    \includegraphics[width=20mm]{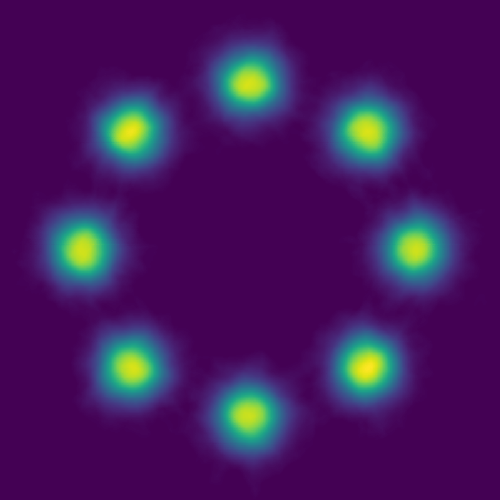}
    & \hspace{0.15cm}
    \includegraphics[width=20mm]{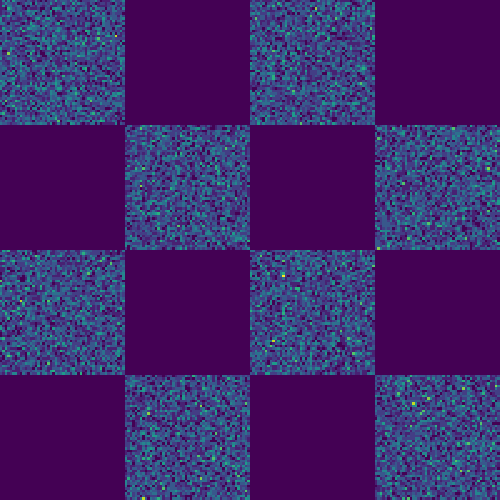}
    &
    \includegraphics[width=20mm]{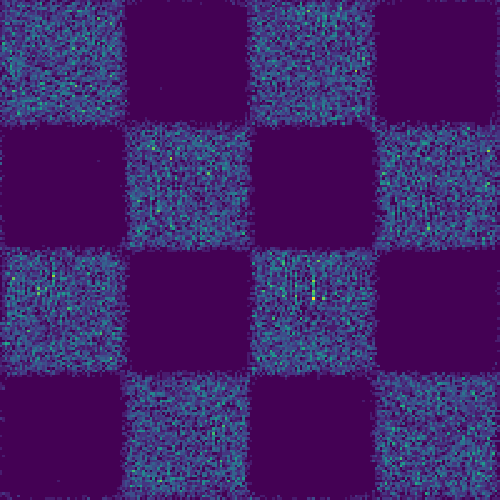}
    &
    \includegraphics[width=20mm]{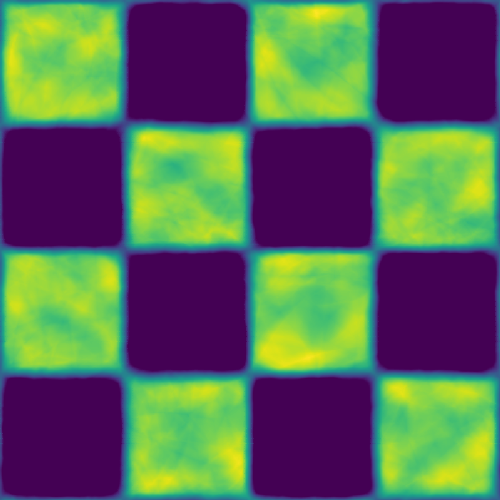} 
    \\
    \includegraphics[width=20mm]{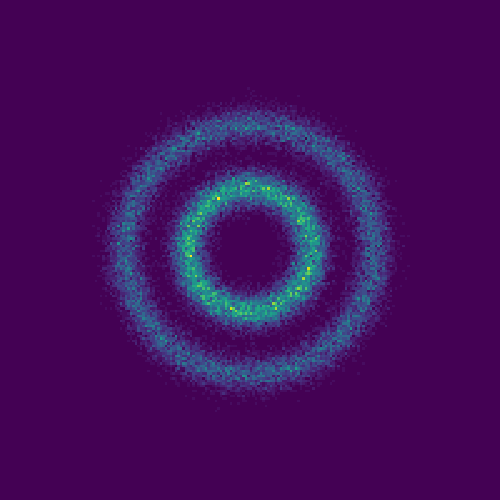}
    &
    \includegraphics[width=20mm]{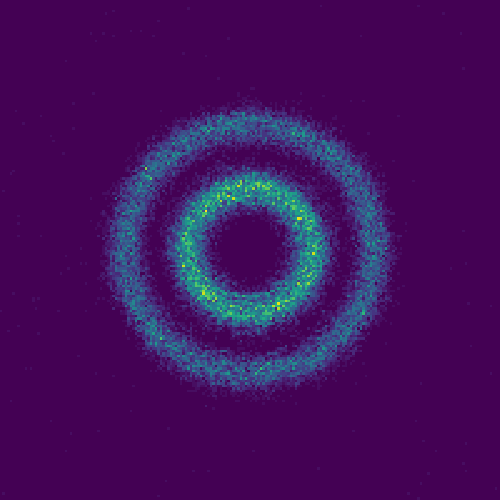}
    &
    \includegraphics[width=20mm]{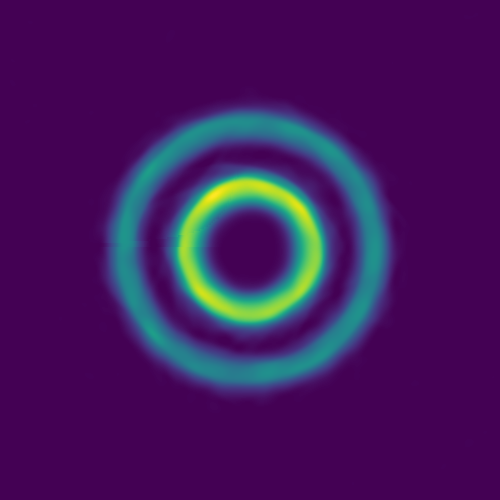} 
    & \hspace{0.15cm}
    \includegraphics[width=20mm]{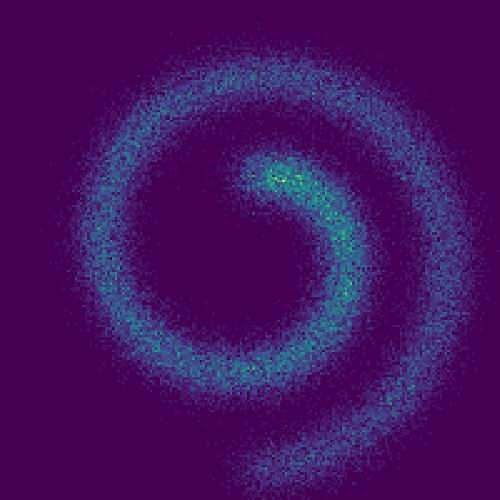}
    &
    \includegraphics[width=20mm]{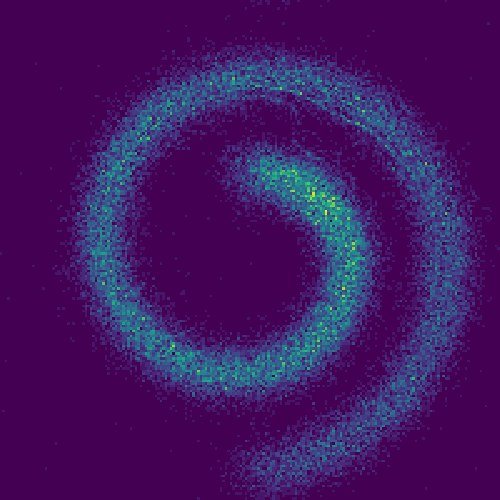}
    &
    \includegraphics[width=20mm]{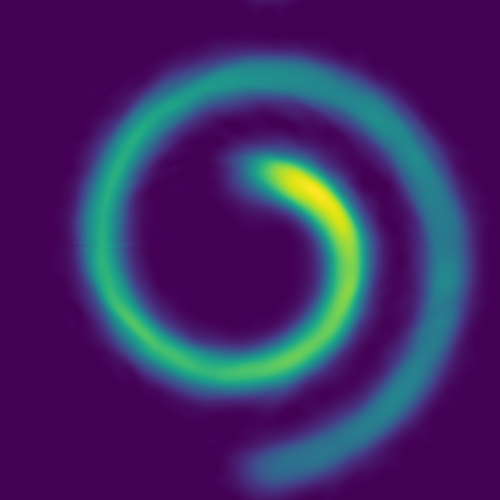} 
    \\
    \includegraphics[width=20mm]{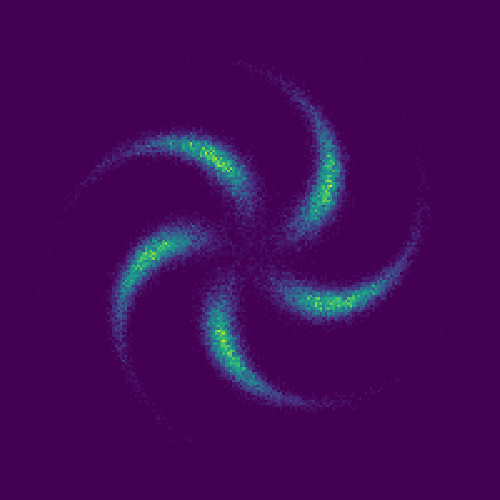}
    &
    \includegraphics[width=20mm]{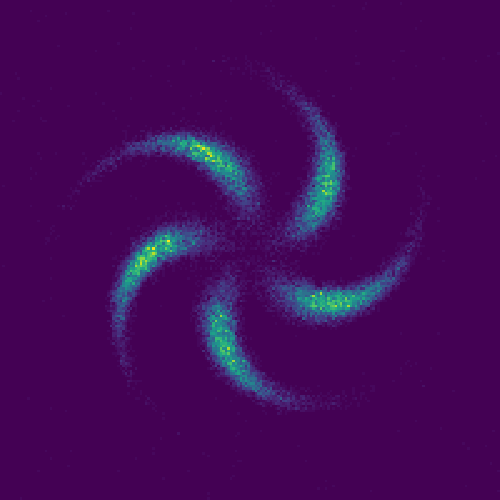}
    &
    \includegraphics[width=20mm]{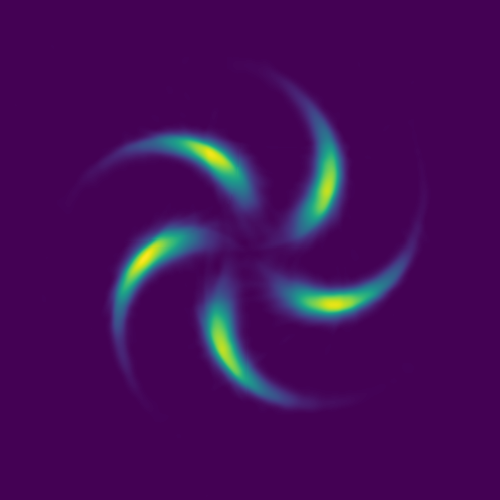} 
    & \hspace{0.15cm}
    \includegraphics[width=20mm]{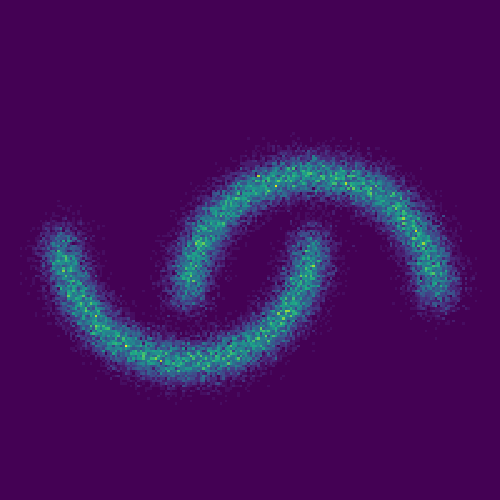}
    &
    \includegraphics[width=20mm]{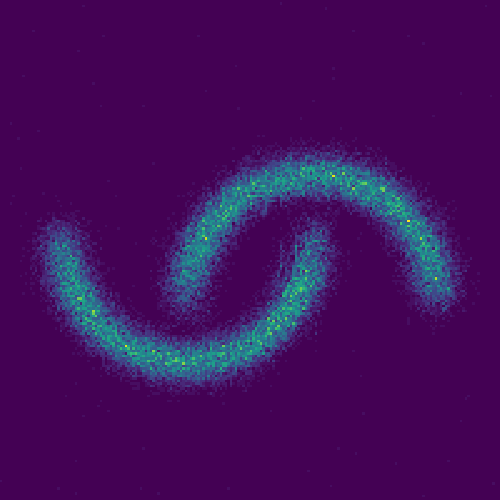}
    &
    \includegraphics[width=20mm]{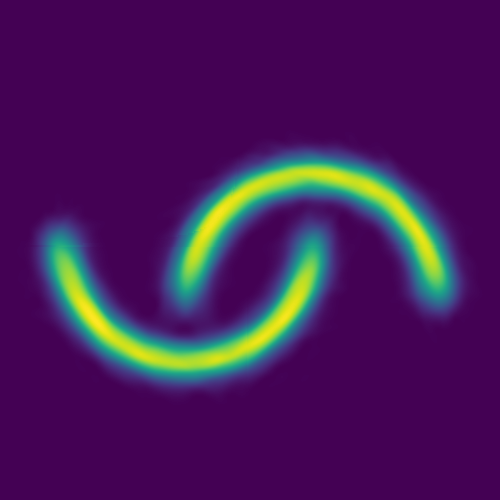} 
    \\
    \includegraphics[width=20mm]{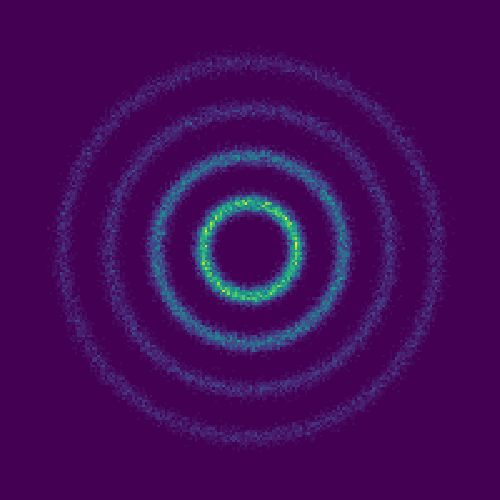}
    &
    \includegraphics[width=20mm]{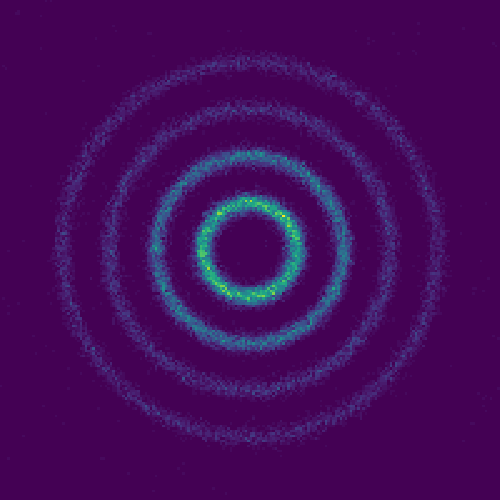}
    &
    \includegraphics[width=20mm]{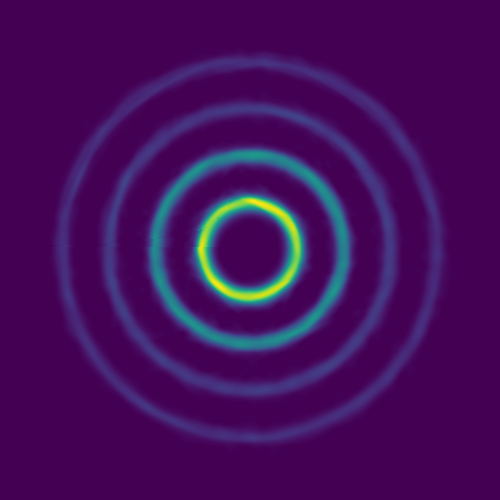} 
    & \hspace{0.15cm} 
    \includegraphics[width=20mm]{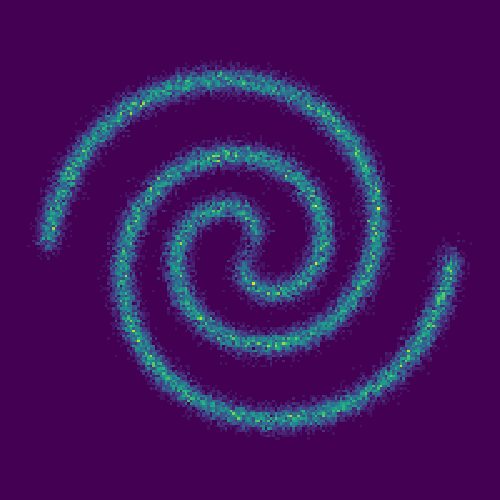}
    &
    \includegraphics[width=20mm]{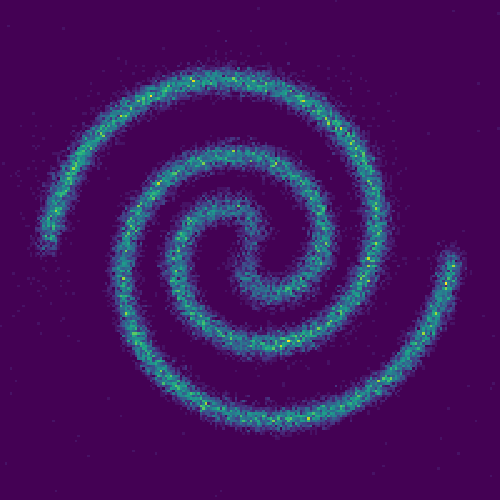}
    &
    \includegraphics[width=20mm]{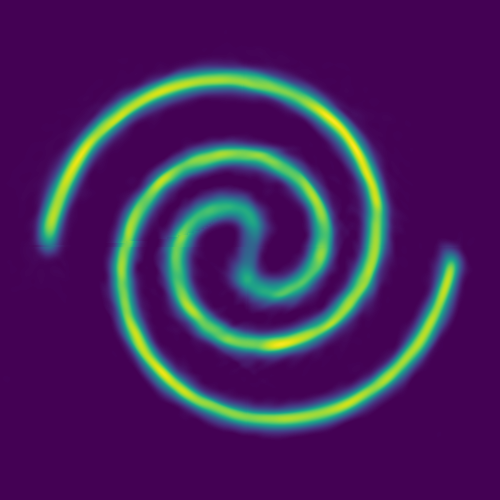} 
\end{tabular}
\caption{Moser Flow trained on 2D datasets. We show generated samples and learned density $\bar{\mu}_+$.\vspace{-10pt} }\label{fig:toy}
\end{figure*}

\section{Experiments}
In all experiments, we modeled a manifold vector field as a multi-layer perceptron (MLP) $\vu_\theta \in \mathfrak{X}(\gM)$, with parameters $\theta$. All models were trained using Adam optimizer \citep{kingma2014adam}, and in all neural networks the activation is Softplus with $\beta=100$. Unless stated otherwise, we set $\lambda_+=0$. We used an exact calculation of the divergence $\mathrm{div}_E(\vu(\vx))$.
%
%
We experimented with two kinds of manifolds.

\textbf{Flat Torus.}
To test our method on Euclidean 2D data, we used $\gM$ as the flat torus, that is the unit square $[-1, 1]^2$ with opposite edges identified. This defines a manifold with no boundary which is locally isometric to the Euclidean plane. Due to this local isometry the Riemannian divergence on the flat torus is equivalent to the Euclidean divergence, $\mathrm{div}=\mathrm{div}_E$. To make $\vu_\theta$ a well defined smooth vector field in $\gM$ we use periodic positional encoding, namely $\vu_\theta(\vx) = \vv_\theta(\tau(\vx))$, where $\vv_\theta$ is a standard MLP and $\tau:\Real^2\too\Real^{4k}$ is defined as $\tau(\vx)=(\cos(\omega_1\pi \vx), \sin(\omega_1\pi \vx), ..., \cos(\omega_k\pi \vx), \sin(\omega_k\pi \vx)) $, where $w_i=i$, and $k$ is a hyper-parameter that is application dependent. 
Since any periodic function can be approximated by a polynomial acting on $e^{i\pi \vx}$, even for $k=1$ this is a universal model for continuous functions on the torus. As described by \cite{tancik2020fourier}, adding extra features can help with learning higher frequencies in the data. To solve an ODE on the torus we simply solve it for the periodic function and wrap the result back to  $[-1, 1]^2$.

\textbf{Implicit surfaces.}
We experiment with surfaces as submanifolds of $\Real^3$. We represent a surface as the zero level set of a Signed Distance Function (SDF) $f:\R^3\rightarrow\R$. We experimented with two surfaces. First, the sphere, represented with the SDF $f(\vx)=\norm{\vx} - 1$, and second, the Stanford Bunny surface, representing a general curved surface and represented with an SDF learned with \citep{gropp2020implicit} from point cloud data. 
To model vector fields on an implicit surface we follow the general \eqref{e:vu}, where for SDFs
\begin{equation*}
    \pi(\vx)= \vx-f(\vx)\nabla f(\vx), \quad \text{and }\ \mP_\vx = \mI - \nabla f(\vx)\nabla f(\vx)^T.
\end{equation*} 
In the supplementary, we detail how to replace the global projection $\pi(\vx)$ with a local one, for cases the SDF is not exact. 
\vspace{-5pt}

\subsection{Toy distributions}\vspace{-5pt}
First, we consider a collection of challenging toy 2D datasets explored in prior works \citep{chen2020residual,huang2021convex}. We scale samples to lie in the flat torus $[-1,1]^2$ and use $k=1$ for the positional encoding. Figure \ref{fig:toy} depicts the input data samples, the generated samples \emph{after} training, and the learned distribution $\bar{\mu}$. In the top six datasets, the MLP architecture used for Moser Flows consists of 3 hidden layers with 256 units each, whereas in the bottom two we used 4 hidden layers with 256 neurons due to the higher complexity of these distributions. We set $\lambda_-=2$.

\begin{figure*}
\centering
\begin{tabular}{@{\hskip0pt}c@{\hskip0pt}c@{\hskip0pt}c@{\hskip 5 pt}cc@{\hskip2.5pt}c@{\hskip0pt}c@{\hskip5pt}c@{\hskip0pt}}
    &$\mu_+$ & $\Phi_*\nu$ & $\abs{\Phi_*\nu - \mu_+}$ & \hspace{0.15cm} &$\mu_+$ & $\Phi_*\nu$ & $\abs{\Phi_*\nu - \mu_+}$ \\
    \rotatebox{90}{\quad $\lambda=1$} &
    \includegraphics[width=20mm]{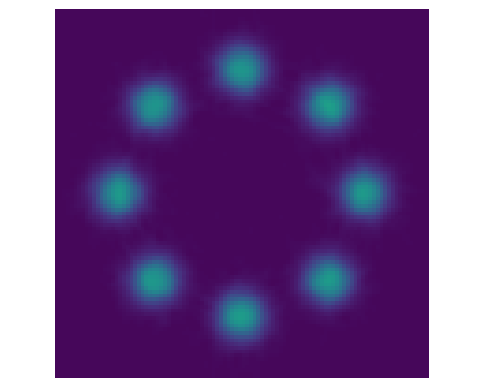}
    &
    \includegraphics[width=20mm]{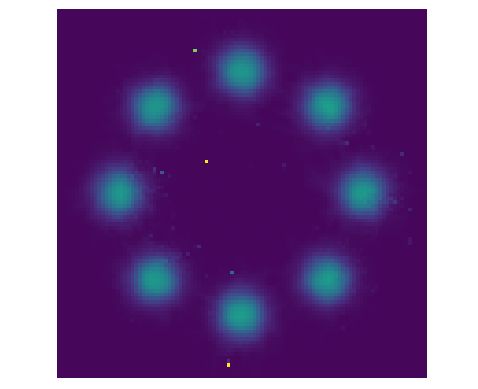}
    &
    \includegraphics[width=20mm]{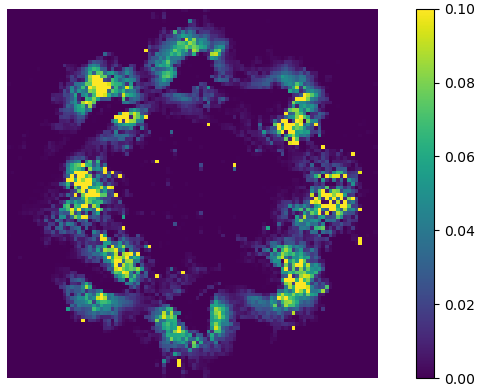}
    & \rotatebox{90}{\quad $\lambda=2$} &
    \includegraphics[width=20mm]{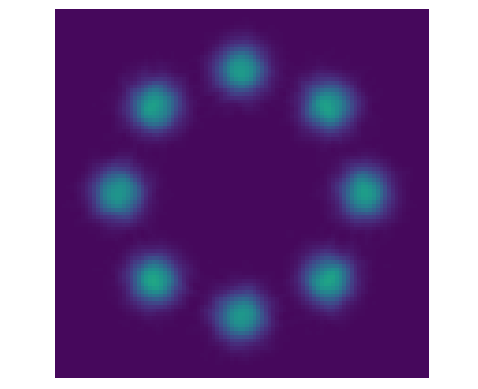}
    &
    \includegraphics[width=20mm]{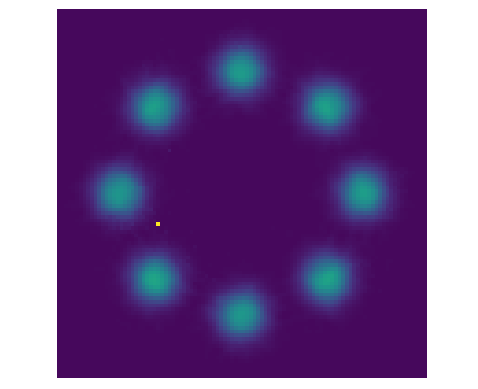}
    &
    \includegraphics[width=20mm]{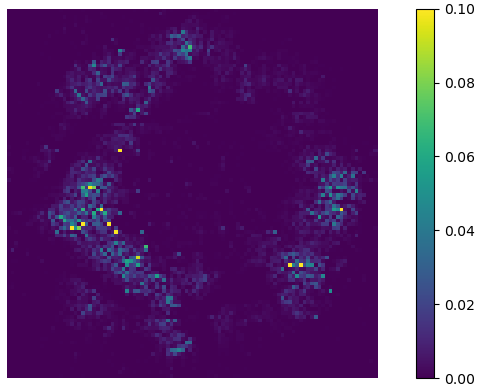}
    \\
    \rotatebox{90}{\quad $\lambda=10$} &
    \includegraphics[width=20mm]{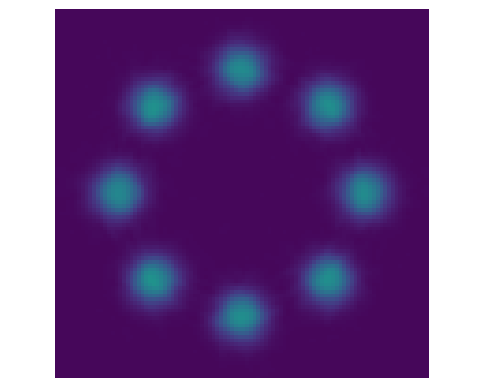}
    &
    \includegraphics[width=20mm]{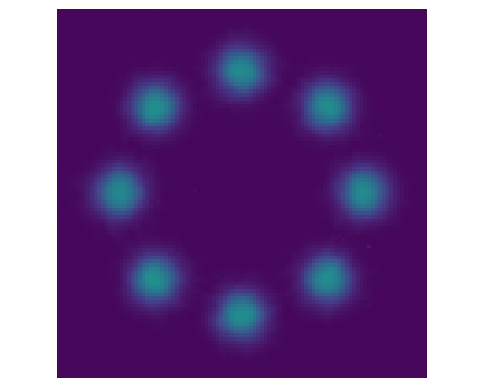}
    &
    \includegraphics[width=20mm]{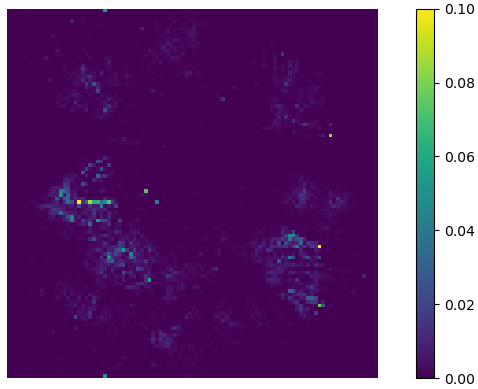}
    & \rotatebox{90}{\quad $\lambda=100$} &
    \includegraphics[width=20mm]{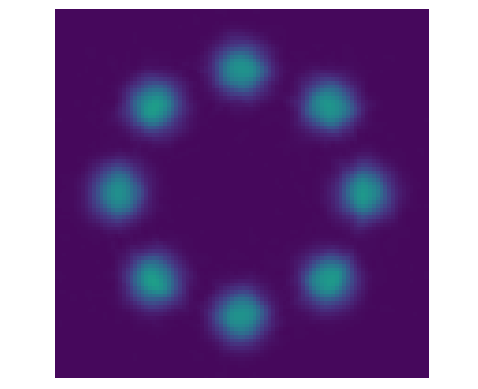}
    &
    \includegraphics[width=20mm]{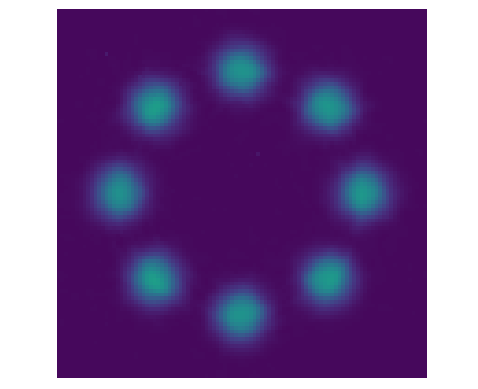}
    &
    \includegraphics[width=20mm]{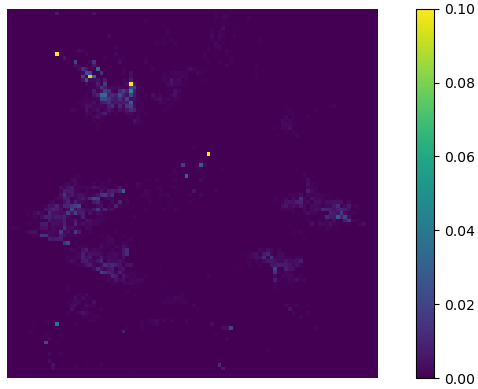}
\end{tabular}
\caption{As $\lambda$ is increased, the closer $\bar{\mu}_+$ is to the generated density $\Phi_*\nu$; column titled $\abs{\bar{\mu}_+ - \Phi_*\nu}$ shows the absolute pointwise difference between the two; note that some of the errors in the $\abs{\bar{\mu}_+ - \Phi_*\nu}$ column are due to numerical inaccuracies in the ODE solver used to calculate $\Phi_*\nu$. }\label{fig:lambdas}
\end{figure*}

\subsection{Choice of hyper-parameter $\lambda$}
We test the effect of the hyper-parameter $\lambda\geq 1$ on the learned density. Figure \ref{fig:lambdas} shows, for different values of $\lambda$, our density estimation $\mu_+$, the push-forward density $\Phi_*\nu$, and their absolute difference. To evaluate $\Phi_*\nu$ from the vector field $v_t$, we solve an ODE as advised in \cite{grathwohl2018ffjord}. 
As expected, higher values of $\lambda$ lead to closer modeled density $\bar{\mu}_+$ and $\Phi_*\nu$. This is due to the fact that a higher value of $\lambda$ leads to a lower value of $\int_\gM\bar{\mu}_-$, meaning $\bar{\mu}$ is a better representation of a probability density. Nonetheless, even for $\lambda=1$ the learned and generated density are rather consistent. \vspace{-5pt}


\begin{figure}[t]
    \centering
    \begin{tabular}{@{\hskip0pt}c@{\hskip1.0pt}c@{\hskip1.5pt}c@{\hskip1.5pt}c@{\hskip1.5pt}c@{\hskip5.0pt}c@{\hskip0.0pt}}
         & density 1k & samples 1k & density 5k & samples 5k & \\
    \rotatebox{90}{\quad Moser Flow} &
    \includegraphics[width=0.17\columnwidth]{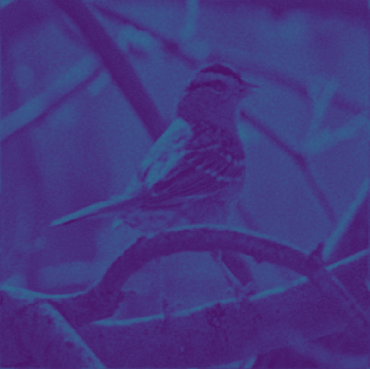}
         &
         \includegraphics[width=0.17\columnwidth]{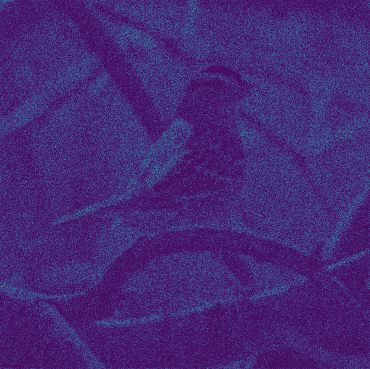}
         & 
         \includegraphics[width=0.17\columnwidth]{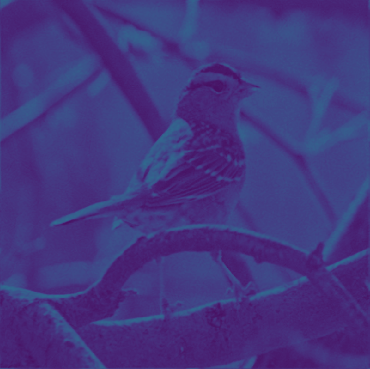}
         &
         \includegraphics[width=0.17\columnwidth]{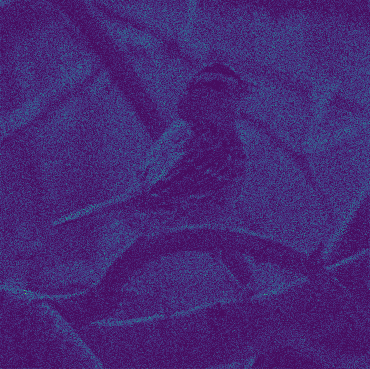}
         & 
         \includegraphics[width=0.25\columnwidth]{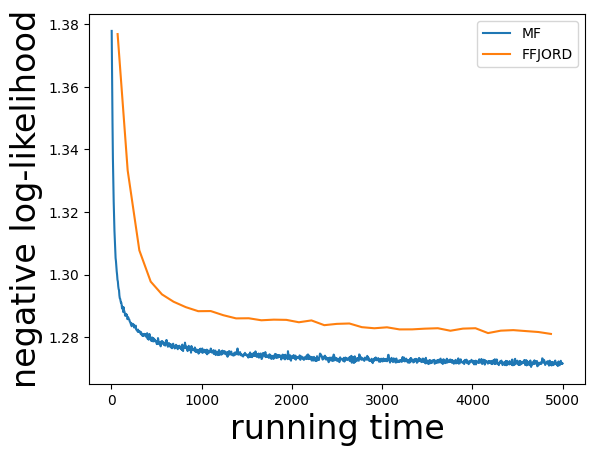} 
         \\
          \rotatebox{90}{\qquad  FFJORD}  &
          \includegraphics[width=0.17\columnwidth]{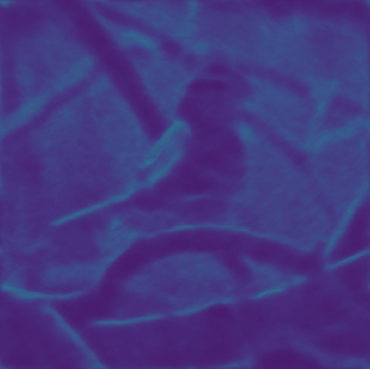}
         &
         \includegraphics[width=0.17\columnwidth]{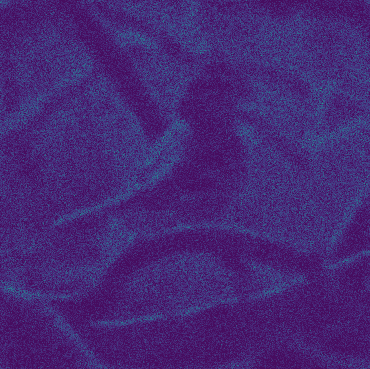}
         & 
         \includegraphics[width=0.17\columnwidth]{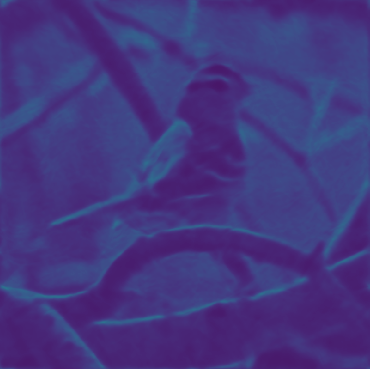}
         &
         \includegraphics[width=0.17\columnwidth]{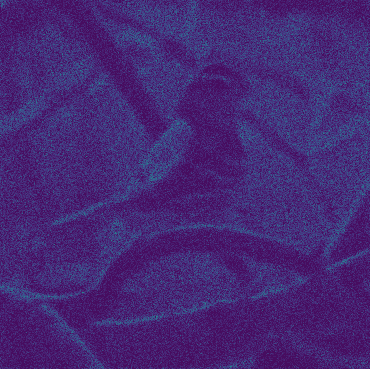}
         & 
         \includegraphics[width=0.25\columnwidth]{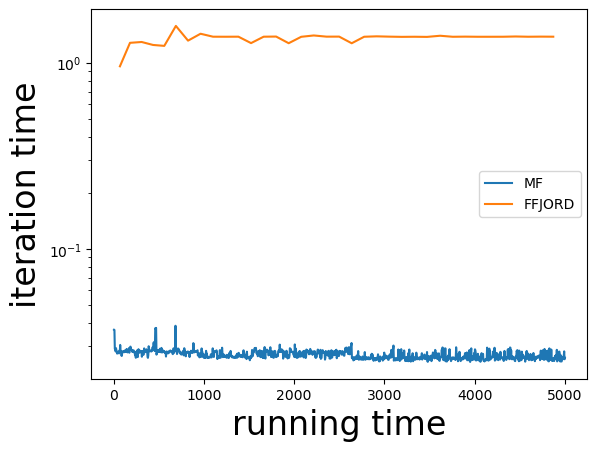}
         \\
    \end{tabular}
    \caption{Comparing learned density and generated samples with MF and FFJORD at different times (in k-sec); top right shows NLL scores for both MF and FFJORD at different times; bottom right shows time per iteration (in $\log$-scale, sec) as a function of total running time (in sec); FFJORD iterations take longer as training progresses; A second example of the same experiment on a different image is provided in the supplementary. Flickr image (license CC BY 2.0): Bird by Flickr user "lakeworth" \url{https://www.flickr.com/photos/lakeworth/46657879995/}. \vspace{-5pt}}
    \label{fig:cameraman}
\end{figure}

\subsection{Time evaluations}\vspace{-5pt}
To compare our method to Euclidean CNF methods, we compare Moser Flow with FFJORD~\citep{grathwohl2018ffjord} on the flat torus. We consider a challenging density with high frequencies obtained via 800x600 images (Figure \ref{fig:cameraman}). We generate a new batch every iteration by sampling each pixel location with probability which is proportional to the pixel intensity. 
The architectures of both $\vv_\theta$ and the vector field defined in FFJORD are the same, namely an MLP with 4 hidden layers of 256 neurons each. To capture the higher frequencies in the image we use a positional encoding with $k=8$ for both methods. We used a batch size of 10k. We used learning rate of 1e-5 for Moser Flow and 1e-4 for FFJORD. We used $\lambda_-=2$. Learning was stopped after 5k seconds. Figure \ref{fig:cameraman} presents the results. Note that Moser Flow captures high-frequency details better than FFJORD. This is also expressed in the graph on the top right, showing how the NLL decreases faster for MF than FFJORD. Furthermore, as can be inspected in the per iteration time graph on the bottom-right, MF per iteration time does not increase during training, and is roughly 1-2 order of magnitudes faster than FFJORD iteration.

\begin{figure}[t]
    \centering
    \begin{tabular}{@{\hskip0.0pt}c@{\hskip0.0pt}c@{\hskip0.0pt}c@{\hskip0.0pt}c@{\hskip0.0pt}}
       \includegraphics[width=0.25\columnwidth]{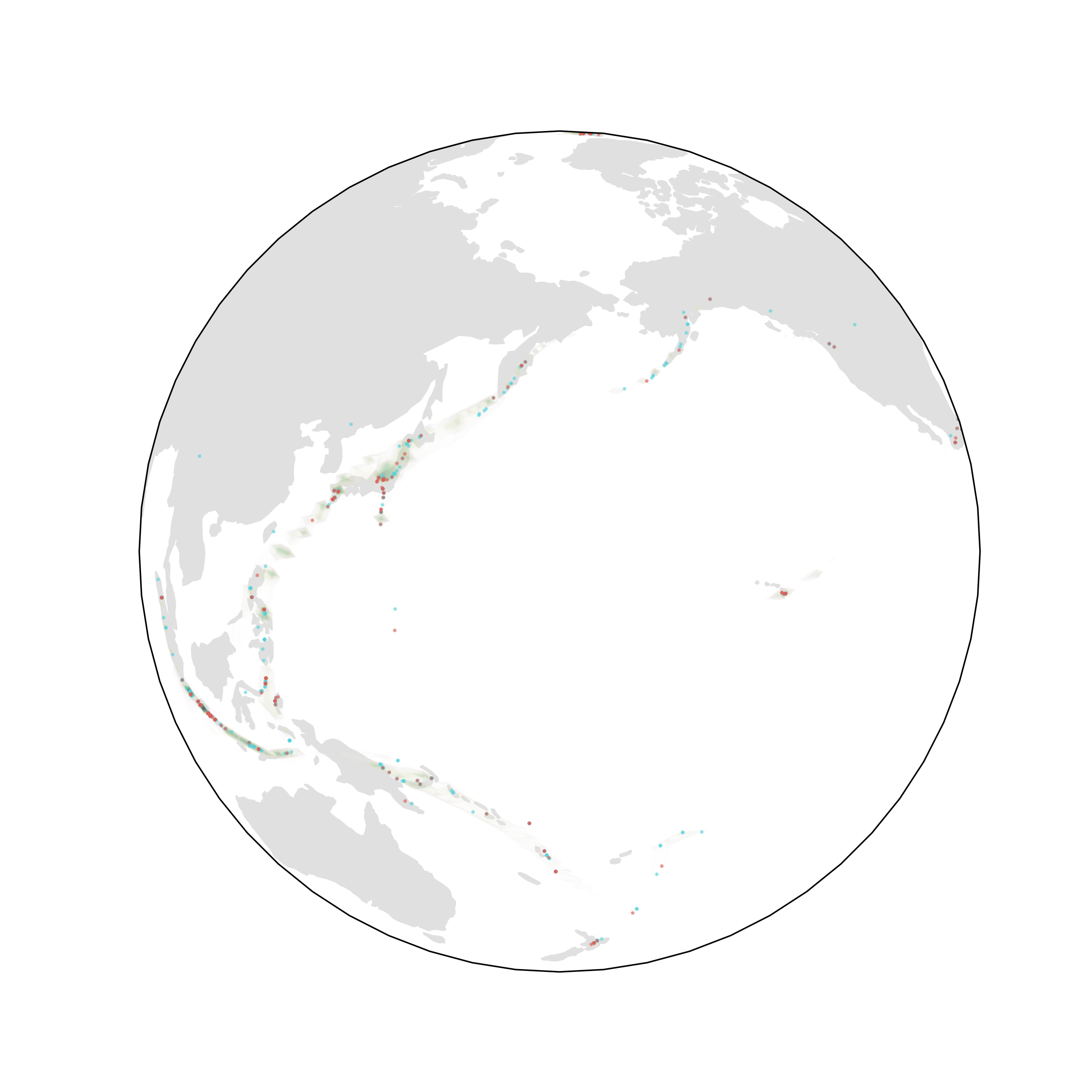}  &  
       \includegraphics[width=0.25\columnwidth]{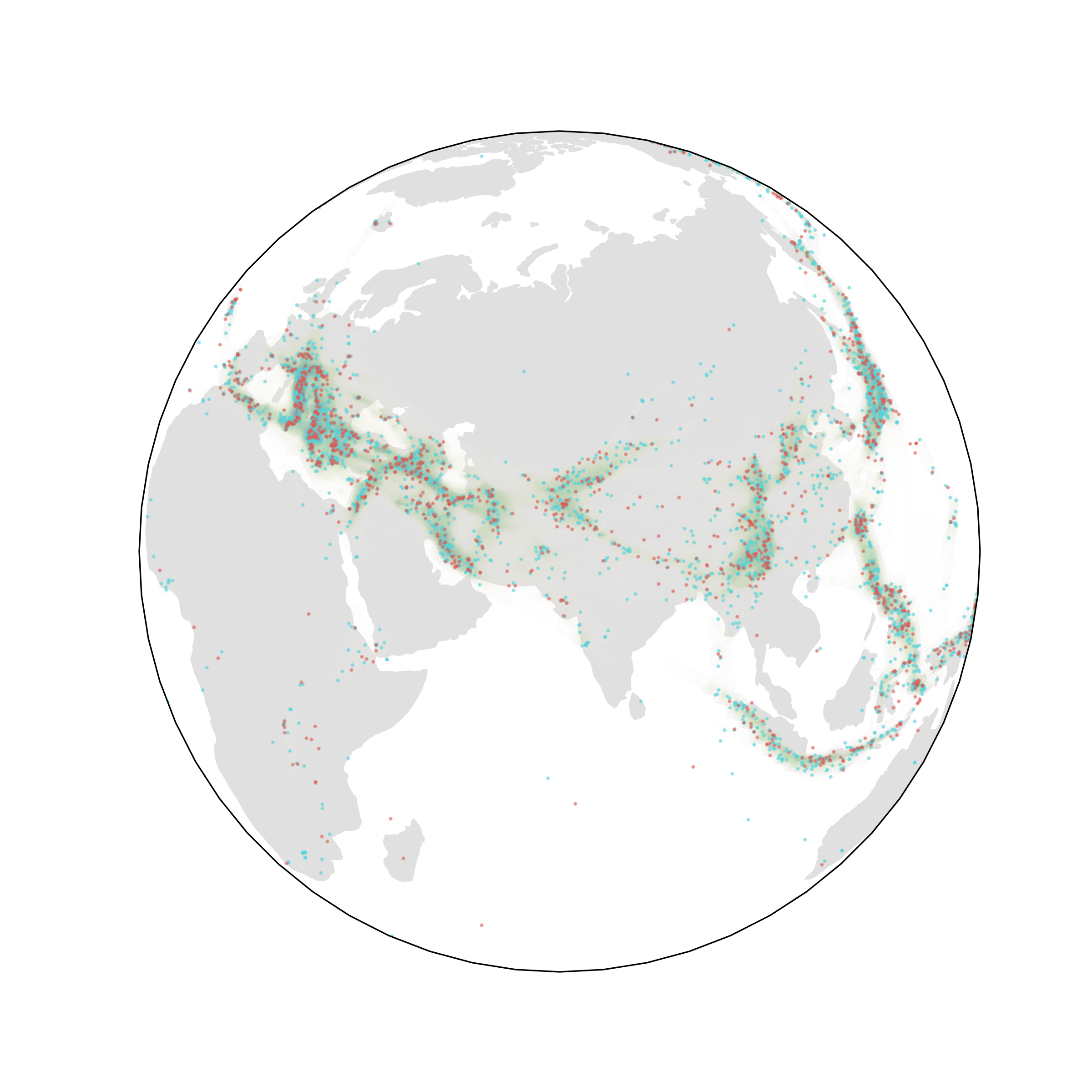} &
       \includegraphics[width=0.25\columnwidth]{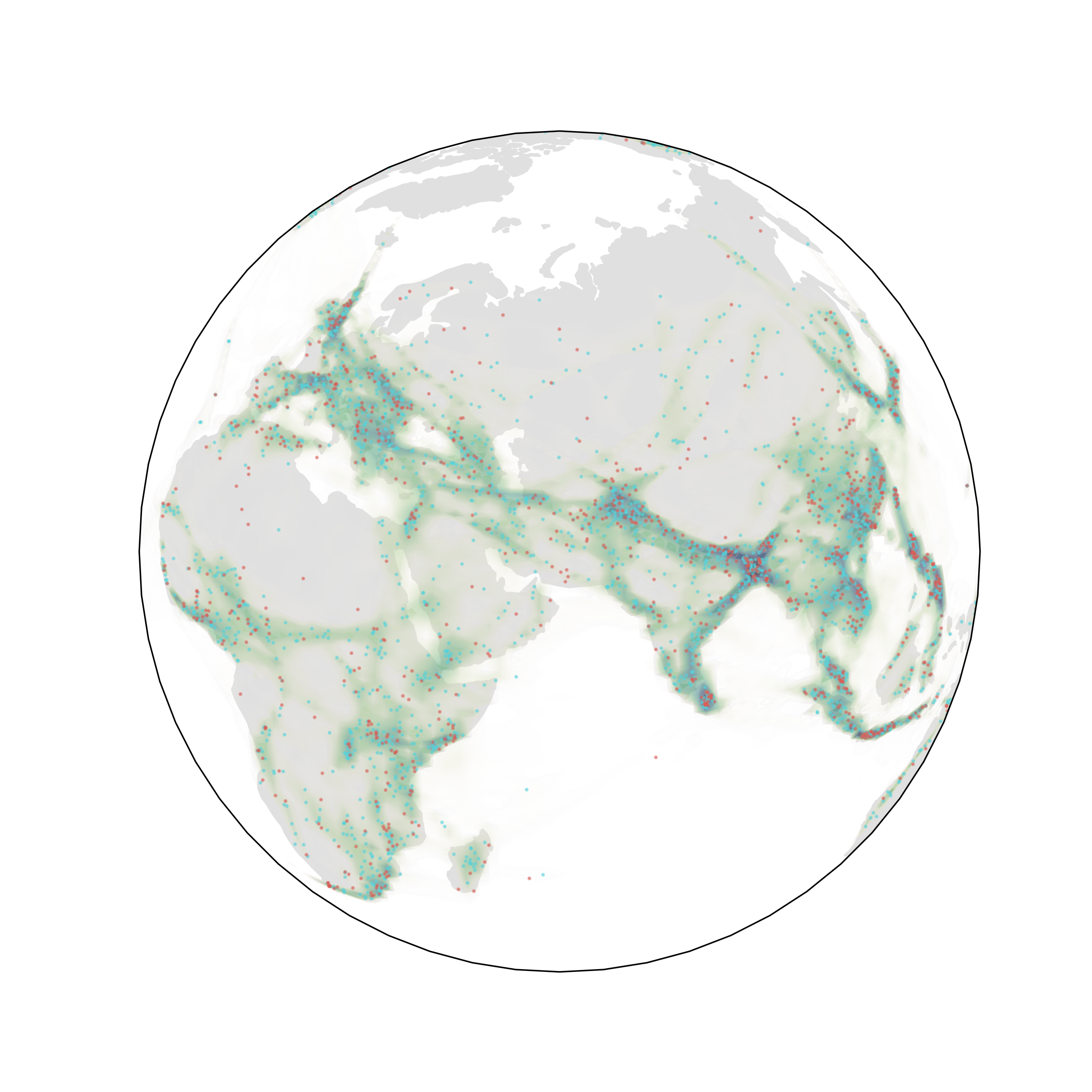} & 
       \includegraphics[width=0.25\columnwidth]{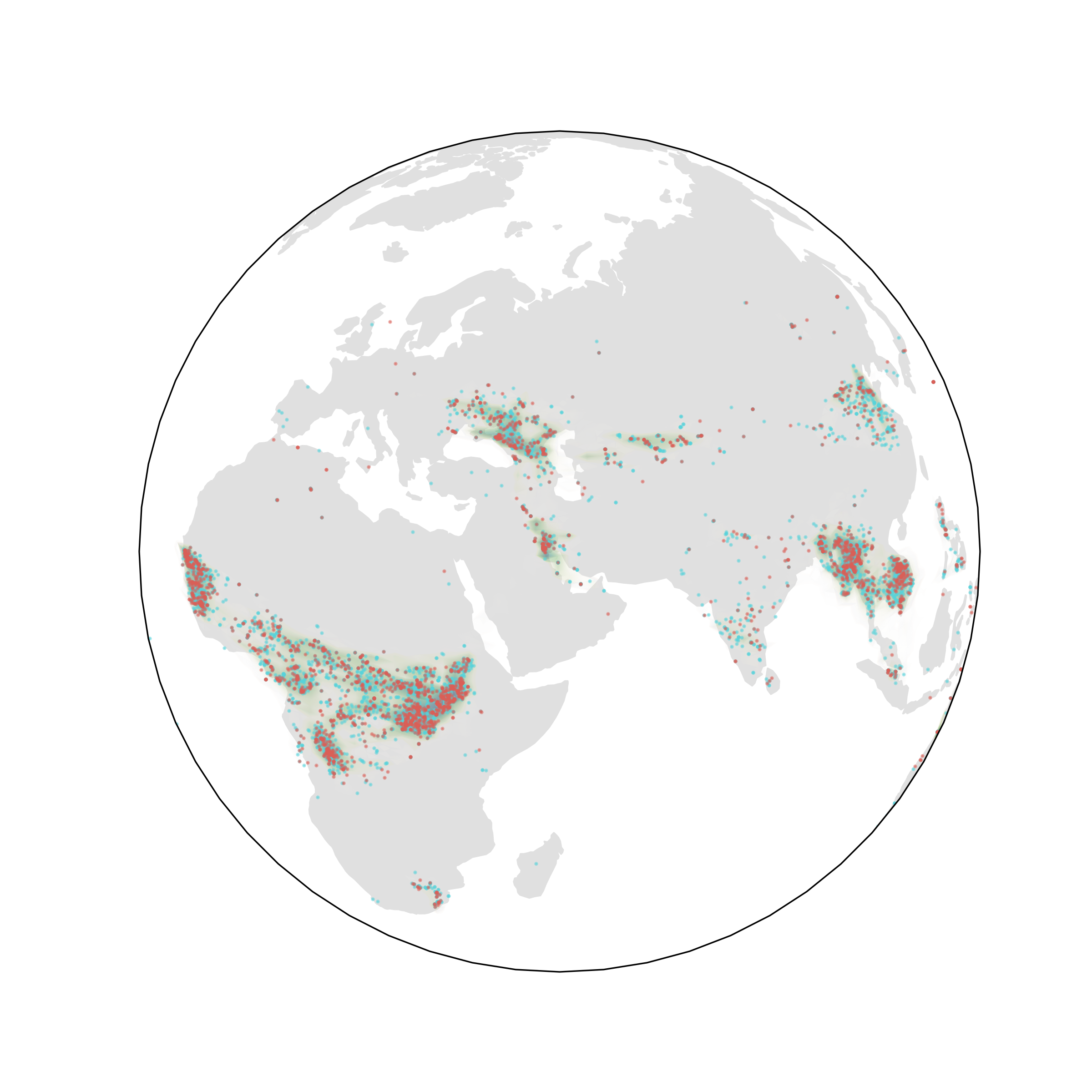} \vspace{-10pt} \\
       Volcano & Earthquake & Flood & Fire
    \end{tabular}
    \caption{Moser Flow trained on earth sciences data gathered by \cite{mathieu2020riemannian}. The learned density is colored green-blue (blue indicates larger values); Blue and red dots represent training and testing datapoints, respectively. See Table \ref{tab:earth} for matching quantitative results. \vspace{-5pt}  }
    \label{fig:earth}
\end{figure}

\begin{table}[t]
    \centering
    \resizebox{0.6\columnwidth}{!}{%
    \begin{tabular}{ccccc}
   & Volcano & Earthquake & Flood & Fire \\
 \hline
 Mixture vMF & $-0.31_{\pm 0.07}$ & $0.59_{\pm 0.01}$ & $1.09_{\pm 0.01}$ & $-0.23_{\pm 0.02}$ \\
 Stereographic & $-0.64_{\pm 0.20}$ & $0.43_{\pm 0.04}$ & $0.99_{\pm 0.04}$ & $-0.40_{\pm 0.06}$ \\
 Riemannian & $-0.97_{\pm 0.15}$ & $0.19_{\pm 0.04}$ & $0.90_{\pm 0.03}$ & $-0.66_{\pm 0.05}$\\
 Moser Flow (MF) & $\mathbf{-2.02}_{\pm 0.42}$ & $\mathbf{-0.09}_{\pm 0.02}$ & $\mathbf{0.62}_{\pm 0.04}$ & $\mathbf{ -1.03}_{\pm 0.03}$ \\
 \hline
 Data size & $829$ & $6124$ & $4877$ & $12810$ \\
 \hline \vspace{1pt}
\end{tabular}}
    \caption{Negative log-likelihood scores of the earth sciences datasets.}
    \label{tab:earth}
\end{table}

\subsection{Earth and climate science data}\vspace{-0pt}
We evaluate our model on the earth and climate datasets gathered in \cite{mathieu2020riemannian}. The projection operator $\pi$ in this case is simply $\pi(\vx)=\frac{\vx}{\norm{\vx}}$. We parameterize $\vv_\theta$ as an MLP with 6 hidden layers of 512 neurons each. We used full batches for the NLL loss and batches of size 150k for our integral approximation. We trained for 30k epochs, with learning rate of 1e-4. We used $\lambda_-=100$. The quantitative NLL results are reported in Table \ref{tab:earth} and qualitative visualizations in \ref{fig:earth}. Note that we produce NLL scores smaller than the runner-up method by a large margin.

\subsection{Curved surfaces}\vspace{-5pt}
We trained an SDF $f$ for the Stanford Bunny surface $\gM$ using the method in \cite{gropp2020implicit}. 
To generate uniform ($\nu$) and data ($\mu$) samples over $\gM$ we first extract a mesh $\gM'$ from $f$ using the Marching Cubes algorithm \citep{lorensen1987marching} setting its resolution to 100x100x100. Then, to randomly choose a point uniformly from $\gM'$ we first randomly choose a face of the mesh with probability proportional to its area, and then randomly choose a point uniformly within that face. For target $\mu$ we used clamped manifold harmonics to create a sequence of densities with increased complexity. To that end, we first computed the $k$-th eigenfunction of the Laplace-Beltrami operator over $\gM'$ (we provide details on this computation in the supplementary), for the frequencies (eigenvalues) $k\in\set{10,50,500}$. Next, we sampled the eigenfunctions at the faces' centers, clamped their negative values, and normalized to get discrete probability densities over the faces of $\gM'$. Then, to sample a point, we first choose a face at random based on this probability, and then random a point uniformly within that face. 
We take 500k i.i.d.~samples of this distribution as our dataset.
We take $\vv_\theta$ to be an MLP with 6 hidden layers of dimension 512. We use batch size of 10k for both the NLL loss and for the integral approximation; we ran for 1000 epochs with learning rate of 1e-4. We used $\lambda_-=\lambda_+=1$. Figure \ref{fig:bunny} depict the results. Note that Moser Flow is able to learn the surface densities for all three frequencies. \vspace{-5pt}

\section{Related Work}\vspace{-5pt}




In the following, we discuss related work on normalizing flows for
manifold-valued data. On a high level, such methods can be divided into
\emph{Parametric} versus \emph{Riemannian} methods. Parametric methods consist of a normalizing flow in the Euclidean space $\R^n$, pushed-forward onto the manifold through an invertible map $\psi:\Real^n \rightarrow \gM$. However, to globally represent the manifold, $\psi$ needs to be a homeomorphism implying  that $\gM$ and $\R^n$ are topologically equivalent, limiting the scope of that
approach. Existing methods in this class are often based on the exponential map
$\exp_x: T_{x}\gM \cong \R^n \to \gM$ of a manifold. This leads to the so called
\emph{wrapped} distributions. This approach has been taken, for instance, by \cite{falorsi2019Reparameterizing} and \cite{bose2020Latent} to parametrize
probability distributions on Lie groups and hyperbolic space. However, Parametric methods based on the exponential map often lead to numerical and computational
challenges. For instance, in compact manifolds (\eg, spheres or the
$\text{SO}(3)$ group) computing the density of \emph{wrapped} distributions
requires an infinite summation. On the hyperboloid, on the other hand, the
exponential map is numerically not well-behaved far away from the origin
\citep{dooley1993Harmonic,al-mohy2010New}.

In contrast to Parametric methods, Riemannian methods operate directly on the
manifold itself and, as such, avoid numerical instabilities that arise from the
mapping onto the manifold. Early work in this class of models proposed transformations along
geodesics on the hypersphere by evaluating the exponential map at the gradient
of a scalar manifold function \citep{sei2011Jacobian}.
\cite{rezende2020normalizing} introduced \emph{discrete} \emph{Riemannian} flows
for hyperspheres and torii based on M\"{o}bius transformations and spherical
splines. \citet{mathieu2020riemannian} introduced \emph{continuous}
flows on general Riemannian manifolds (RCNF). In contrast to discrete flows
\citep[\eg,][]{bose2020Latent,rezende2020normalizing}, such time-continuous
flows alleviate the previous topological constraints by parametrizing the flow as the solution to an ODE over the manifold \citep{grathwohl2018ffjord}.
Concurrently to RCNF, \cite{lou2020neural} and \cite{falorsi2020Neural}
proposed related extensions of neural ODEs to smooth manifolds. Moser Flow also generates a CNF, however by limiting the flow space (albeit, not the generated distributions) it allows expressing the learned distribution as the divergence of a vector field. \vspace{-5pt}

\begin{figure}[t]
    \centering
    \begin{tabular}{@{\hskip0.0pt}c@{\hskip0.0pt}c@{\hskip1.0pt}|@{\hskip1.0pt}c@{\hskip0.0pt}c@{\hskip1.0pt}|@{\hskip1.0pt}c@{\hskip0.0pt}c@{\hskip0.0pt}}
       \includegraphics[width=0.16\columnwidth]{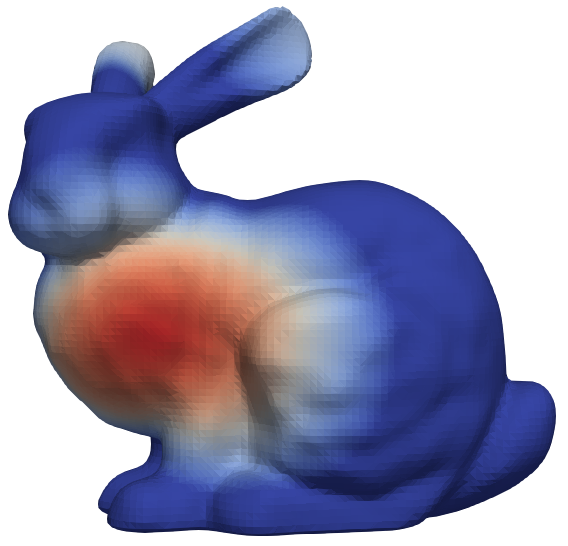}  &
       \includegraphics[width=0.16\columnwidth]{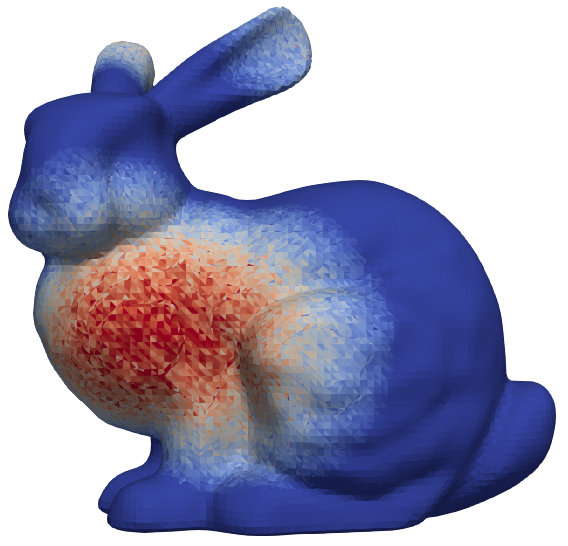}  &
       \includegraphics[width=0.16\columnwidth]{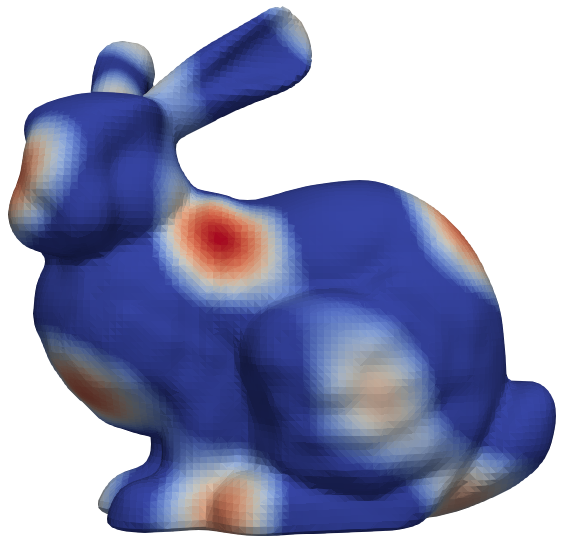}  &
       \includegraphics[width=0.16\columnwidth]{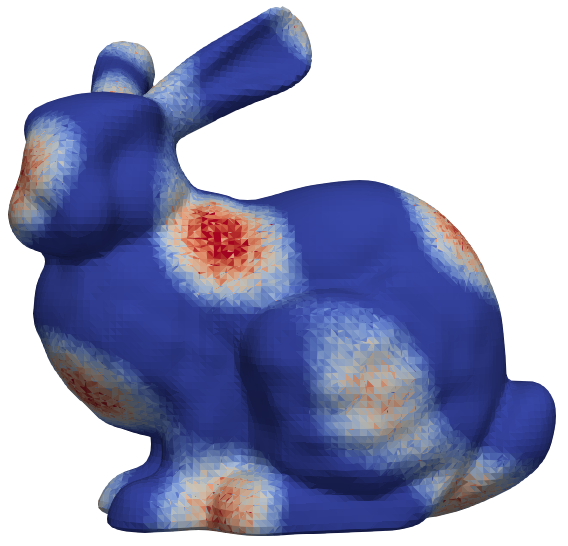}  &
       \includegraphics[width=0.16\columnwidth]{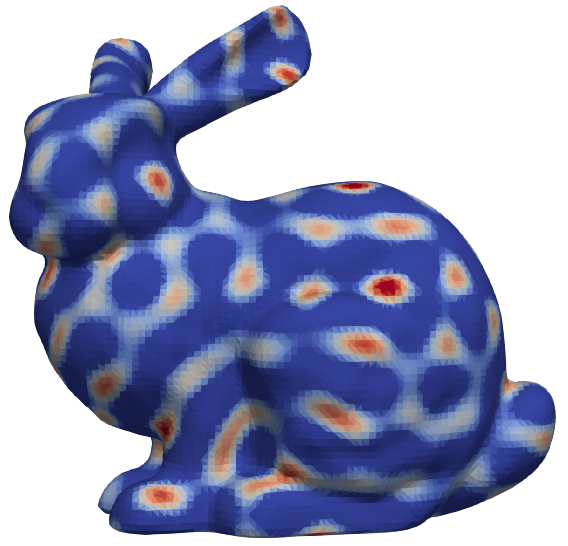}  &
       \includegraphics[width=0.16\columnwidth]{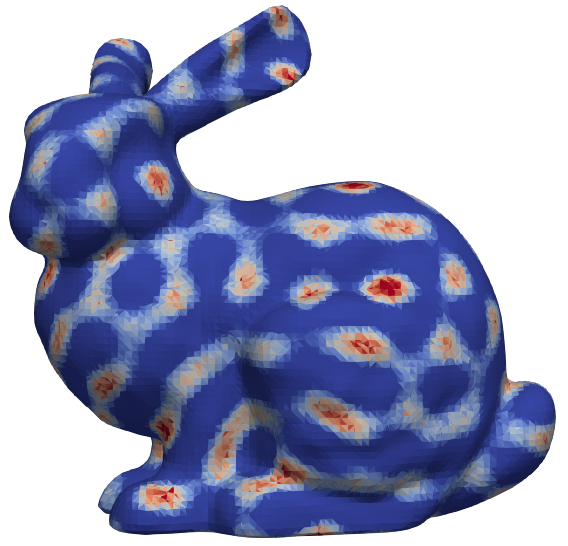}  \\
       \includegraphics[width=0.16\columnwidth]{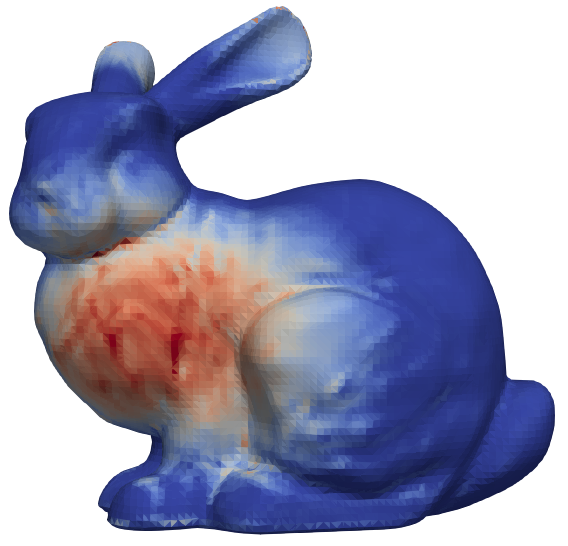} &
       \includegraphics[width=0.16\columnwidth]{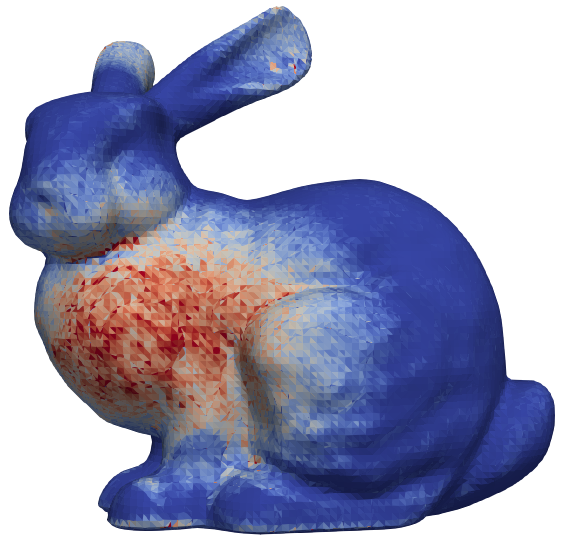}  &
       \includegraphics[width=0.16\columnwidth]{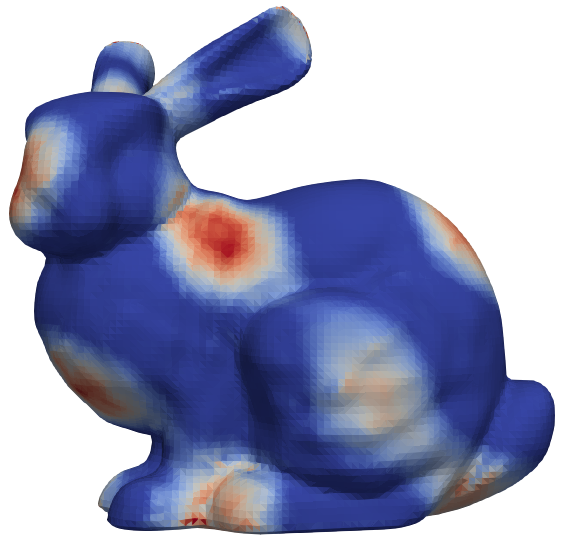}  &
       \includegraphics[width=0.16\columnwidth]{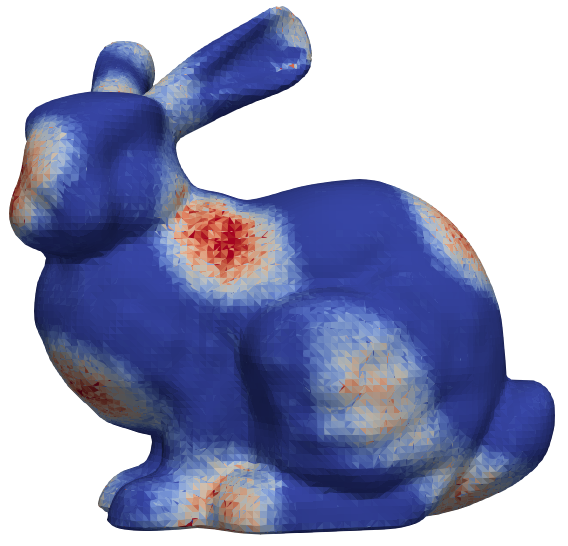}  &
       \includegraphics[width=0.16\columnwidth]{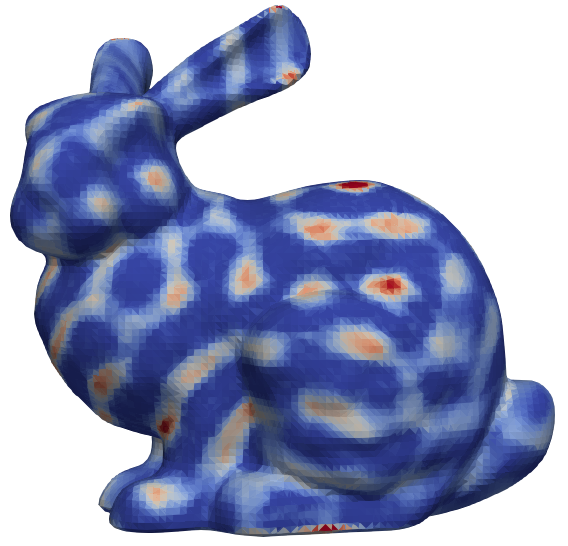}  &
       \includegraphics[width=0.16\columnwidth]{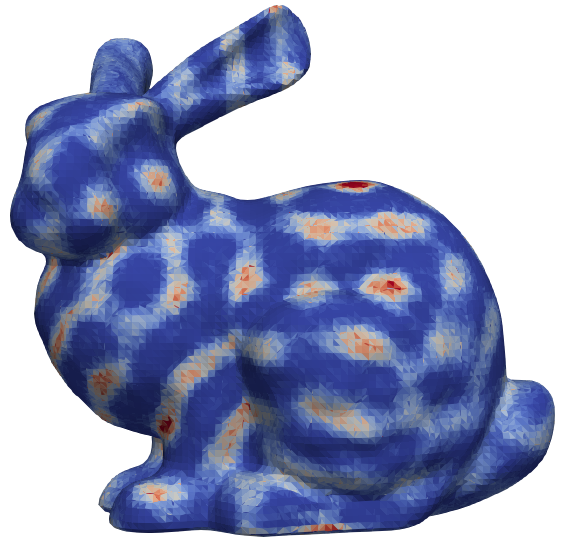}  \\
       \multicolumn{2}{c}{Frequency $k=10$} & \multicolumn{2}{c}{Frequency $k=50$} & \multicolumn{2}{c}{Frequency $k=500$}
    \end{tabular}
    \caption{Moser Flow trained on a curved surface (Stanford Bunny). We show three different target distribution with increasing frequencies, where for each frequency we depict (clockwise from top-left): target density, data samples, generated samples, and learned density. \vspace{-10pt}}
    \label{fig:bunny}
\end{figure}

\section{Discussion and limitations}\vspace{-5pt}
We introduced Moser Flow, a generative model in the family of CNFs that represents the target density using the divergence operator applied to a vector valued neural network. The main benefits of MF stems from the simplicity and locality of the divergence operator. MFs circumvent the need to solve an ODE in the training process, and are thus applicable on a broad class of manifolds. Theoretically, we prove MF is a universal generative model, able to (approximately) generate arbitrary positive target densities from arbitrary positive prior densities. Empirically, we show MF enjoys favorable computational speed in comparison to previous CNF models, improves density estimation on spherical data compared to previous work by a large margin, and for the first time facilitate training a CNF over a general curved surface.

One important future work direction, and a current limitation, is scaling of MF to higher dimensions. This challenge can be roughly broken to three parts: First, the model probabilities $\bar{\mu}$ should be computed/approximated in log-scale, as probabilities are expected to decrease exponentially with the dimension. Second, the variance of the numerical approximations of the integral $\int_\gM \bar{\mu}_- d\vol$ will increase significantly in high dimensions and needs to be controlled. Third, the divergence term, $\mathrm{div}(u)$, is too costly to be computed exactly in high dimensions and needs to be approximated, similarly to other CNF approaches.  
%
Finally, our work suggests a novel generative model, and similarly to other generative models can be potentially used for generation of fake data and amplify harmful biases in the dataset. Mitigating such harms is an active and important area of ongoing research.

\newpage
\section*{Acknowledgments}
NR is supported by the European Research Council (ERC Consolidator Grant, "LiftMatch" 771136), the Israel Science Foundation (Grant No. 1830/17), and Carolito Stiftung (WAIC).

\bibliography{2021_moser}

\begin{thebibliography}{}

\bibitem[{Al-Mohy} and Higham, 2010]{al-mohy2010New}
{Al-Mohy}, A.~H. and Higham, N.~J. (2010).
\newblock A {{New Scaling}} and {{Squaring Algorithm}} for the {{Matrix
  Exponential}}.
\newblock {\em SIAM Journal on Matrix Analysis and Applications},
  31(3):970--989.

\bibitem[Bose et~al., 2020]{bose2020Latent}
Bose, A.~J., Smofsky, A., Liao, R., Panangaden, P., and Hamilton, W.~L. (2020).
\newblock Latent {{Variable Modelling}} with {{Hyperbolic Normalizing Flows}}.
\newblock {\em arXiv:2002.06336 [cs, stat]}.

\bibitem[Botsch et~al., 2010]{botsch2010polygon}
Botsch, M., Kobbelt, L., Pauly, M., Alliez, P., and L{\'e}vy, B. (2010).
\newblock {\em Polygon mesh processing}.
\newblock CRC press.

\bibitem[Brown et~al., 2020]{brown2020language}
Brown, T.~B., Mann, B., Ryder, N., Subbiah, M., Kaplan, J., Dhariwal, P.,
  Neelakantan, A., Shyam, P., Sastry, G., Askell, A., et~al. (2020).
\newblock Language models are few-shot learners.
\newblock {\em arXiv preprint arXiv:2005.14165}.

\bibitem[Chen et~al., 2012]{chen2012triangulated}
Chen, M., Tu, B., and Lu, B. (2012).
\newblock Triangulated manifold meshing method preserving molecular surface
  topology.
\newblock {\em Journal of Molecular Graphics and Modelling}, 38:411--418.

\bibitem[Chen et~al., 2018]{chen2018neural}
Chen, R.~T., Rubanova, Y., Bettencourt, J., and Duvenaud, D. (2018).
\newblock Neural ordinary differential equations.
\newblock {\em arXiv preprint arXiv:1806.07366}.

\bibitem[Chen et~al., 2020]{chen2020residual}
Chen, R. T.~Q., Behrmann, J., Duvenaud, D., and Jacobsen, J.-H. (2020).
\newblock Residual flows for invertible generative modeling.

\bibitem[Dacorogna and Moser, 1990]{dacorogna1990partial}
Dacorogna, B. and Moser, J. (1990).
\newblock On a partial differential equation involving the jacobian
  determinant.
\newblock In {\em Annales de l'Institut Henri Poincare (C) Non Linear
  Analysis}, volume~7, pages 1--26. Elsevier.

\bibitem[Dhariwal and Nichol, 2021]{dhariwal2021diffusion}
Dhariwal, P. and Nichol, A. (2021).
\newblock Diffusion models beat gans on image synthesis.
\newblock {\em arXiv preprint arXiv:2105.05233}.

\bibitem[Dinh et~al., 2016]{dinh2016density}
Dinh, L., Sohl-Dickstein, J., and Bengio, S. (2016).
\newblock Density estimation using real nvp.
\newblock {\em arXiv preprint arXiv:1605.08803}.

\bibitem[Do~Carmo, 2016]{do2016differential}
Do~Carmo, M.~P. (2016).
\newblock {\em Differential geometry of curves and surfaces: revised and
  updated second edition}.
\newblock Courier Dover Publications.

\bibitem[Dooley and Wildberger, 1993]{dooley1993Harmonic}
Dooley, A. and Wildberger, N. (1993).
\newblock Harmonic analysis and the global exponential map for compact {{Lie}}
  groups.
\newblock {\em Functional Analysis and Its Applications}, 27(1):21--27.

\bibitem[Falorsi et~al., 2019]{falorsi2019Reparameterizing}
Falorsi, L., {de Haan}, P., Davidson, T.~R., and Forr{\'e}, P. (2019).
\newblock Reparameterizing {{Distributions}} on {{Lie Groups}}.
\newblock {\em arXiv:1903.02958 [cs, math, stat]}.

\bibitem[Falorsi and Forr{\'e}, 2020]{falorsi2020Neural}
Falorsi, L. and Forr{\'e}, P. (2020).
\newblock Neural {{Ordinary Differential Equations}} on {{Manifolds}}.
\newblock {\em arXiv:2006.06663 [cs, stat]}.

\bibitem[Gerber et~al., 2010]{gerber2010manifold}
Gerber, S., Tasdizen, T., Fletcher, P.~T., Joshi, S., Whitaker, R., Initiative,
  A. D.~N., et~al. (2010).
\newblock Manifold modeling for brain population analysis.
\newblock {\em Medical image analysis}, 14(5):643--653.

\bibitem[Grathwohl et~al., 2018]{grathwohl2018ffjord}
Grathwohl, W., Chen, R. T.~Q., Bettencourt, J., Sutskever, I., and Duvenaud, D.
  (2018).
\newblock Ffjord: Free-form continuous dynamics for scalable reversible
  generative models.

\bibitem[Gropp et~al., 2020]{gropp2020implicit}
Gropp, A., Yariv, L., Haim, N., Atzmon, M., and Lipman, Y. (2020).
\newblock Implicit geometric regularization for learning shapes.

\bibitem[Hoppe et~al., 1992]{hoppe1992surface}
Hoppe, H., DeRose, T., Duchamp, T., McDonald, J., and Stuetzle, W. (1992).
\newblock Surface reconstruction from unorganized points.
\newblock In {\em Proceedings of the 19th annual conference on computer
  graphics and interactive techniques}, pages 71--78.

\bibitem[Hornik et~al., 1990]{hornik1990universal}
Hornik, K., Stinchcombe, M., and White, H. (1990).
\newblock Universal approximation of an unknown mapping and its derivatives
  using multilayer feedforward networks.
\newblock {\em Neural networks}, 3(5):551--560.

\bibitem[Huang et~al., 2021]{huang2021convex}
Huang, C.-W., Chen, R. T.~Q., Tsirigotis, C., and Courville, A. (2021).
\newblock Convex potential flows: Universal probability distributions with
  optimal transport and convex optimization.

\bibitem[Kazhdan et~al., 2006]{kazhdan2006poisson}
Kazhdan, M., Bolitho, M., and Hoppe, H. (2006).
\newblock Poisson surface reconstruction.
\newblock In {\em Proceedings of the fourth Eurographics symposium on Geometry
  processing}, volume~7.

\bibitem[Kingma and Ba, 2014]{kingma2014adam}
Kingma, D.~P. and Ba, J. (2014).
\newblock Adam: A method for stochastic optimization.
\newblock {\em arXiv preprint arXiv:1412.6980}.

\bibitem[Kumar et~al., 2019]{kumar2019videoflow}
Kumar, M., Babaeizadeh, M., Erhan, D., Finn, C., Levine, S., Dinh, L., and
  Kingma, D. (2019).
\newblock Videoflow: A flow-based generative model for video.
\newblock {\em arXiv preprint arXiv:1903.01434}, 2(5).

\bibitem[Lang, 2012]{lang2012fundamentals}
Lang, S. (2012).
\newblock {\em Fundamentals of differential geometry}, volume 191.
\newblock Springer Science \& Business Media.

\bibitem[Lee, 2013]{lee2013smooth}
Lee, J.~M. (2013).
\newblock Smooth manifolds.
\newblock In {\em Introduction to Smooth Manifolds}, pages 1--31. Springer.

\bibitem[Lorensen and Cline, 1987]{lorensen1987marching}
Lorensen, W.~E. and Cline, H.~E. (1987).
\newblock Marching cubes: A high resolution 3d surface construction algorithm.
\newblock {\em ACM siggraph computer graphics}, 21(4):163--169.

\bibitem[Lou et~al., 2020]{lou2020neural}
Lou, A., Lim, D., Katsman, I., Huang, L., Jiang, Q., Lim, S.-N., and De~Sa, C.
  (2020).
\newblock Neural manifold ordinary differential equations.

\bibitem[Mathieu and Nickel, 2020]{mathieu2020riemannian}
Mathieu, E. and Nickel, M. (2020).
\newblock Riemannian continuous normalizing flows.
\newblock {\em arXiv preprint arXiv:2006.10605}.

\bibitem[Morita, 2001]{morita2001geometry}
Morita, S. (2001).
\newblock {\em Geometry of differential forms}.
\newblock Number 201. American Mathematical Soc.

\bibitem[Moser, 1965]{moser1965volume}
Moser, J. (1965).
\newblock On the volume elements on a manifold.
\newblock {\em Transactions of the American Mathematical Society},
  120(2):286--294.

\bibitem[Papamakarios et~al., 2019]{papamakarios2019normalizing}
Papamakarios, G., Nalisnick, E., Rezende, D.~J., Mohamed, S., and
  Lakshminarayanan, B. (2019).
\newblock Normalizing flows for probabilistic modeling and inference.
\newblock {\em arXiv preprint arXiv:1912.02762}.

\bibitem[Rezende and Mohamed, 2015]{rezende2015variational}
Rezende, D. and Mohamed, S. (2015).
\newblock Variational inference with normalizing flows.
\newblock In {\em International Conference on Machine Learning}, pages
  1530--1538. PMLR.

\bibitem[Rezende et~al., 2020]{rezende2020normalizing}
Rezende, D.~J., Papamakarios, G., Racaniere, S., Albergo, M., Kanwar, G.,
  Shanahan, P., and Cranmer, K. (2020).
\newblock Normalizing flows on tori and spheres.
\newblock In {\em International Conference on Machine Learning}, pages
  8083--8092. PMLR.

\bibitem[Sei, 2011]{sei2011Jacobian}
Sei, T. (2011).
\newblock A {{Jacobian Inequality}} for {{Gradient Maps}} on the {{Sphere}} and
  {{Its Application}} to {{Directional Statistics}}.
\newblock {\em Communications in Statistics - Theory and Methods},
  42(14):2525--2542.

\bibitem[Tancik et~al., 2020]{tancik2020fourier}
Tancik, M., Srinivasan, P.~P., Mildenhall, B., Fridovich-Keil, S., Raghavan,
  N., Singhal, U., Ramamoorthi, R., Barron, J.~T., and Ng, R. (2020).
\newblock Fourier features let networks learn high frequency functions in low
  dimensional domains.

\end{thebibliography}
\bibliographystyle{apalike} 

\appendix
\section*{\Large{Supplementary Material}}

\appendix
\section{Proof of Moser's Theorem.}  We will review here the proof of Moser Theorem \ref{thm:moser}; for more details see Moser's original paper \citep{moser1965volume} or \cite{lang2012fundamentals}, Chapter 18 section 2.
Let $\hat{\alpha}_t=\alpha_t dV$ be the time-dependent volume form over $\gM$ corresponding to the density interpolant $\alpha_t$. Note that $\int_\gM \hat{\alpha}_t=1$.
Moser's idea is to replace \eqref{e:normalizing} with its continuous version: 
\begin{equation}\label{e:cont_normalization}
    \hat{\alpha}_0 = \Phi_t^* \hat{\alpha}_t, \quad t\in [0,1]
\end{equation}
If \eqref{e:cont_normalization} holds for all $t\in[0,1]$ then plugging $t=1$ leads to \eqref{e:normalizing}. Since \eqref{e:cont_normalization} holds trivially for $t=0$ (since $\Phi_0$ is the identity mapping), solving it amounts to asking that $\Phi^*_t \hat{\alpha}_t$ is constant, \ie, 
\begin{equation}\label{e:dt_alpha_t}
\frac{d}{dt}\Phi_t^* \hat{\alpha}_t = 0.    
\end{equation}
The time derivative of $\Phi^*_t\hat{\alpha}_t$ can be computed with the help of the Lie derivative (\eg, Proposition 5.2 in \cite{lang2012fundamentals}): If $\Phi_t$ is the flow corresponding to the time dependent vector field $v_t$ (see \eqref{e:Phi_ode}), and $\omega$ is a differential form then $$\frac{d}{dt}(\Phi_t^* \omega) = \Phi_t^*(\mathfrak{L}_{v_t}\omega),$$ where $\mathfrak{L}$ denotes the Lie derivative. The Lie derivative $\mathfrak{L}_v\omega$ of a smooth vector field $v$ and smooth differential form $\omega$ can be computed using Cartan's "magic formula" (see \eg, Theorem 14.35 in \cite{lee2013smooth}): $$\mathfrak{L}_v \omega = i_v (d\omega) + d(i_v \omega),$$
where $i_v \omega$ is the interior multiplication of a vector field and a differential form defined by $(i_v \omega)(v_2,\ldots,v_n) = \omega(v,v_2,\ldots,v_n)$. In case $\omega$ is an $n$-form (as $\hat{\alpha}_t$ in our case) we have $d\omega=0$ so the first term in the r.h.s.~above vanishes. 
Lastly, we will need the following "trick": 
$$\frac{d}{dt}(\Phi_t^* \hat{\alpha}_t) = \frac{d}{ds}\Big\vert_{s=t}(\Phi_s^* \hat{\alpha}_t) + \frac{d}{ds}\Big\vert_{s=t}(\Phi_t^* \hat{\alpha}_s).$$
Putting the last three equations together we get:
\begin{equation}\label{e:cartan}
  \frac{d}{dt}(\Phi_t^* \hat{\alpha}_t) = \Phi_t^*(\mathfrak{L}_{v_t}\hat{\alpha}_t) + \Phi_t^* \parr{\frac{d}{dt}\hat{\alpha}_t}=\Phi^*_t \parr{d( i_{v_t} \hat{\alpha}_t ) + \frac{d}{dt}\hat{\alpha}_t   }.  
\end{equation}
The theorem is proven if one can show that $v_t\in\mathfrak{X}(\gM)$ exists such that
    $d( i_{v_t} \hat{\alpha}_t ) + \frac{d}{dt}\hat{\alpha}_t=0$.
The divergence operator is defined by the equality $d(i_{w}d\vol)=\mathrm{div}(w)d\vol$, for a vector field $w\in\mathfrak{X}(\gM)$. Therefore $d( i_{v_t} \hat{\alpha}_t ) = \mathrm{div}(\alpha_t v_t)d\vol$. Denote $\hat{\gamma}_t=\frac{d}{dt}\hat{\alpha}_t$. Then we need to show that $v_t\in\mathfrak{\gM}$ exists such that
\begin{equation}\label{e:need_to_show}
    d( i_{v_t} \hat{\alpha}_t ) + \hat{\gamma}_t=0.
\end{equation}

By the Hodge decomposition (see Theorem 4.18 in \cite{morita2001geometry}) $\hat{\gamma}_t$ can be written as a sum of an exact and harmonic forms: $\hat{\gamma}_t = d \hat{\beta}_t + \hat{h}_t$. Since every harmonic form on a connected, compact, oriented Riemannian manifold is a constant multiple of the Riemannian volume form, $c dV$ (see Corollary 4.14 in \cite{morita2001geometry}), we have
$$0=\frac{d}{dt}1=\frac{d}{dt}\int_\gM \hat{\alpha}_t=\int_\gM \hat{\gamma}_t = \int_\gM d\hat{\beta}_t + \int_\gM \hat{h}_t = \int_\gM \hat{h}_t=c\int_\gM dV,$$
where in the second from the right equality we used Stokes Theorem (see \eg, Theorem 16.11 in  \cite{lee2013smooth}) and the fact that $\gM$ has no boundary. This implies that $c=0$, and
\begin{equation}\label{e:gamma_t}
    \hat{\gamma}_t=d\hat{\beta}_t.
\end{equation}
Using the correspondence between vector fields and $d-1$ forms we let $\beta_t = i_{u_t} dV$, where $u_t\in \mathfrak{X}(\gM)$, and $d\beta_t=d(i_{u_t}dV)=\mathrm{div}(u_t)dV$. 

Lastly, consider $v_t$ defined as follows:
\begin{equation}\label{e:proof_vt}
    v_t = -\frac{u_t}{\alpha_t}.
\end{equation}
With this choice \eqref{e:need_to_show} is satisfied: 
$$d(i_{v_t}\hat{\alpha}_t) + \hat{\gamma}_t = -d(i_{\frac{u_t}{\alpha_t}}(\alpha_t d\vol)) + i_{u_t}d\vol = 0.$$
The theorem is proven.$\qed$

One comment is that for practically finding $v_t$, according to \eqref{e:proof_vt}, we need to get $u_t$, which amounts to solving the Hodge decomposition equation, $\mathrm{div}(u_t)d\vol = \hat{\gamma}_t$, that is equivalent to the following PDE on the manifold $\gM$:
\begin{equation}\label{e:pde_proof}
    \mathrm{div}(u_t) = \frac{d}{dt}\alpha_t.
\end{equation}

\begin{proof}[Proof of Lemma \ref{lem:int_mu}]
The proof uses Stokes theorem: $$\int_\gM \mathrm{div}(u)d\vol = \int_\gM d(i_{u}d\vol) = \int_{\partial\gM} i_{u}d\vol=0,$$
where the last equality is due to the fact that either $\partial \gM=\emptyset$, or, for $x\in\partial \gM$, we have that $u(x)\in T_x \partial\gM$, and therefore $(i_u d\vol)(v_1,\ldots,v_{n-1})=d\vol(u,v_1,\ldots,v_{n-1})=0$, for all $v_1,\ldots,v_{n-1}\in T_x\partial \gM$. This implies $i_u d\vol = 0$. 
\end{proof}

\section{Other proofs}
\begin{proof}[Proof of Theorem \ref{thm:min}]
As we showed in the paper, our loss can be equivalently presented (up to constant factors) as
\begin{equation*}
    l(\theta)=D(\mu, \bar{\mu}_+) + (\lambda-1)\int_\gM\bar{\mu}_-d\vol
\end{equation*}
Where the first term $D(\mu, \bar{\mu}_+)$ is the generalized KL divergence which is non-negative and equals zero iff $\bar{\mu}_+=\mu$ and since $\lambda\geq 1$ the second term is also non-negative and equals zero iff $\mu_-=0$ or $\lambda=1$.\\
First we show that $\bar{\mu}=\mu$ is a minimizer of the loss. Since we assumed $\mu\geq\epsilon$ we have that $\bar{\mu}_+=\max(\mu, \epsilon)=\mu$ and $\bar{\mu}_-=\bar{\mu}_+-\bar{\mu}=0$. So both $D(\mu, \bar{\mu}_+)$ and $\int_\gM\bar{\mu}_-d\vol$ are minimized, which means the entire loss is minimized.\\
Now lets assume $\bar{\mu}$ is a minimizer of the loss. If $\lambda > 1$ $\bar{\mu}$ has to minimize both terms, as we know there exists a minimizer that minimizes both of them. In particular for any $\lambda\geq 1$ we have that $\bar{\mu}$ minimizes $D(\mu, \bar{\mu}_+)$ meaning $\bar{\mu}_+=\mu$. Now we have that $0=1-1=\int_\gM\bar{\mu} d\vol-\int_\gM\mu d\vol=\int_\gM\bar{\mu}_+d\vol + \int_\gM \bar{\mu}_-d\vol-\int_\gM\mu d\vol= \int_\gM \bar{\mu}_-d\vol$. So we get that $\mu_-=0$. Finally $\bar{\mu}=\bar{\mu}_++\bar{\mu}_-=\mu+0=\mu$.
\end{proof}
\begin{proof}[Proof of Lemma \ref{lem:div}.] Proposition 1.2 in \cite{lang2012fundamentals} and Definition 1 in Section 4-4 in \cite{do2016differential} imply that for submanifolds with induced metric the Riemannian covariant derivative at $\vx\in\gM$ satisfies $\nabla_{\ve_i} u = \mP_\vx \frac{\partial \vu}{\partial \ve_i}$, where $\mP_\vx$ is the projection matrix on $T_\vx\gM$ introduced above. Then, denoting $\ve_1,\ldots,\ve_n,\vn_1,\ldots,\vn_k$ an orthonormal basis of $\Real^d$ where the first $n$ vectors span $T_\vx\gM$ and the latter $k$ span $N_\vx\gM$:
\begin{align*}
    \mathrm{div}(\vu) &= \sum_{i=1}^n \ip{\nabla_{\ve_i}\vu,\ve_i}_g = \sum_{i=1}^n \ip{\mP_\vx \frac{\partial\vu}{\partial\ve_i},\ve_i} = \sum_{i=1}^n \ip{ \frac{\partial\vu}{\partial\ve_i},\mP_\vx\ve_i} = \sum_{i=1}^n \ip{ \frac{\partial\vu}{\partial\ve_i},\ve_i}\\ &= \sum_{i=1}^n \ip{ \frac{\partial\vu}{\partial\ve_i},\ve_i} + \sum_{j=1}^k \ip{ \frac{\partial\vu}{\partial\vn_j},\vn_j} = \mathrm{div}_E(\vu),
\end{align*}
\end{proof}

\begin{proof}[Proof of Theorem \ref{thm:universality}.]
From Theorem 6.24 in \cite{lee2013smooth} there exists a neighbourhood $\Omega\subset \Real^d$ of $\gM$ such that the projection $\pi:\Omega\too\gM$ is smooth over $\bar{\Omega}$ (\ie, can be extended to a smooth function over a neighborhood of $\bar{\Omega}$). Since $\gM$ is compact, $\bar{\Omega}$ is also compact. 
According to Theorem \ref{thm:moser} there exists a vector field $\vu^\star \in \mathfrak{X}(\gM)$ so that $\mu = \nu - \mathrm{div}(\vu^\star)$. We extend $\vu^\star$ to $\bar{\Omega}$ by setting  $\vu^\star(\vx) = \vu^\star(\pi(\vx))$, for $\vx\notin \gM$. Note that for $\vx\in\gM$ this definition coincides with the former $\vu^\star$ defined over $\gM$. Similarly to \eqref{e:vu} we have that $\vu^\star(\vx) = \mP_{\pi(\vx)}\vu^\star(\pi(\vx))$.

Corollary 3.4 in \cite{hornik1990universal} shows that given a target smooth function $f:\bar{\Omega}\too\Real$ and $\eps>0$, there exists an MLP with $l$-finite smooth activation that uniformly approximate the first $0\leq m \leq l$ derivatives of $f$ over $\bar{\Omega}$ with error at most $\eps$. An activation $\sigma:\Real\too\Real$ is $l$-finite if it is $l$-times continuously differentiable and satisfies $0<\int_{-\infty}^\infty \abs{\sigma^{(l)}}<\infty$. Note that sigmoid and $\tanh$ are $l$-finite for all $l\geq 1$, and Softplus is $l$-finite for $l\geq 2$. 

Using this approximation result (adapted to vector valued MLP) there exists an MLP $\vv_\theta:\Real^d\too\Real^d$ such that each coordinate of $\vu^\star$ and $\vv_\theta$ are $\eps$ close in value and first partial derivatives over $\bar{\Omega}$. 

Now for arbitrary $\vx\in\gM$ we have
\begin{align*}
    \bar{\mu}(\vx) &= \nu(\vx) -  \mathrm{div}_E(\mP_{\pi(\vx)} \vv_\theta(\pi(\vx)))\\ & = \nu(\vx) -  \mathrm{div}_E\Big (\mP_{\pi(\vx)} \vv_\theta(\pi(\vx))-\mP_{\pi(\vx)} \vu^\star(\pi(\vx))\Big ) - \mathrm{div}(\vu^\star(\vx)) \\
    & = \mu(\vx) - \mathrm{div}_E\Big(\mP_{\pi(\vx)}\brac{ \vv_\theta(\pi(\vx)) - \vu^\star(\pi(\vx)) } \Big) \\
    & = \mu(\vx) - \mathrm{div}_E \Big(\mP_{\pi(\vx)} \ve(\vx) \Big ),
\end{align*}
where we denote $\ve(\vx) = \vv_\theta(\pi(\vx)) - \vu^\star(\pi(\vx))$. We will finish the proof by showing that $$\abs{\mathrm{div}_E \Big(\mP_{\pi(\vx)} \ve(\vx) \Big )} < c\eps$$ for some constant $c>0$ depending only on $\gM$. Note that the l.h.s.~of this equation is a sum of terms of the form $\frac{\partial}{\partial x^i} \parr{(\mP_{\pi(\vx)})_{i,j}\ve(\vx)_j}$, where $(\mP_{\pi(\vx)})_{i,j}$ is the $(i,j)$-th entry of the matrix $\mP_{\pi(\vx)}$ and $\ve(\vx)_j$ is the $j$-th entry of $\ve(\vx)$. Since the value and first partial derivatives of $\pi$ and $\mP$ (as the differential of $\pi$) over $\gM$ can be bounded, depending only on $\gM$, the theorem is proved. 


\end{proof}

\section{Laplacian eigen function calculation}
Given a triangular surface mesh $\gM'$, we wish to calculate the $k$-th eigenfunction of the (discrete) Laplace-Beltrami operator over $\gM'$. We will use the standard (cotangent) discretization of the Laplacian over meshes \citep{botsch2010polygon}. That is, we define $\mL$ to be the cotangent-Laplacian matrix of the graph defined by $\gM'$, and $\mM$ the mass matrix of $\gM'$, \ie, a diagonal matrix where $\mM_{ii}$ is the area of the the Voroni cell of the $i$-th vertex in the mesh. We then calculate the eigenfunctions as the solution to the generalized eigenvalue problem $\mL \vx = \lambda_k \mM \vx$ where $\lambda_k$ is the $k$-th eigenvalue. We sample these $\gM'$ piecewise-linear functions at centroids of faces. 

\section{Linearization of the projection operator $\pi$}
Since we only sample and derivate the projection operator $\pi:\Real^d \too \gM$ over $\gM$, implementing \eqref{e:vu} does not require knowledge of the full projection $\pi$. Rather, it is enough to use its first order expansion over $\gM$. For $\vx_0\in\gM$
$$\pi(\vx)\approx \pi(\vx_0) + \mP_{\vx_0}(\vx-\vx_0) = \vx_0 + \mP_{\vx_0}(\vx-\vx_0)=\hat{\pi}(\vx_0,\vx).$$
Now since $\pi(\cdot)$ and $\hat{\pi}(\vx_0,\cdot)$ have the same value and first partial derivatives at $\vx_0$ we can replace \eqref{e:vu} for each sample point $\vx_0\in\gX\cup \gY$, with
$$\vu(\vx) = \mP_{\hat{\pi}(\vx_0,\vx)}\vv_\theta(\hat{\pi}(\vx_0,\vx)).$$

\section{Unnormalized densities}

As described in section 4, our formulation of the loss is dependent on knowing the volume of the manifold $\gM$. For simple cases like the flat torus or the sphere, we have a closed form formula for this volume. For more general cases, we can show that we don't actually require to know this value, since we can work with unnormalized density functions:
\begin{align*}
    \ell(\theta) &= -\frac{1}{m}\sum_{i=1}^m \log \max\set{\eps,\nu(\vx_i)-\mathrm{div}_E\vu(\vx_i)} \\ & \qquad  + \frac{\vol(\gM)\lambda_-}{l}\sum_{j=1}^l \Big( \eps - \min\set{\eps,\nu(\vy_j)-\mathrm{div}_E\vu(\vy_j)} \Big ), \\
    &= \log \vol(\gM) -\frac{1}{m}\sum_{i=1}^m \log \max\set{\eps',\nu'(\vx_i)-\mathrm{div}_E\vu'(\vx_i)} \\ & \qquad  + \frac{\lambda_-}{l}\sum_{j=1}^l \Big( \eps' - \min\set{\eps',\nu'(\vy_j)-\mathrm{div}_E\vu'(\vy_j)} \Big ), 
\end{align*}
where $\nu'=\vol(\gM)\nu\equiv 1$, $\vu'=\vol(\gM)\vu$, $\eps'=\vol(\gM)\eps'$, and $\log \vol(\gM)$ is a constant. 
Lastly note that the definition of $\vv_t$ is invariant to this scaling and can be computed with the unnormalized quantities.  


\section{Additional Experimental Details}

We used an internal academic cluster with NVIDIA Quadro RTX 6000 GPUs. Every run and seed configuration required 1 GPU.
All other experimental details are mentioned in the main paper.
Our codebase, implemented in PyTorch, is attached in the supplementary materials. We will open-source it post the review process.

\begin{figure}[t]
    \centering
    \begin{tabular}{c@{\hskip1.0pt}c@{\hskip1.5pt}c@{\hskip1.5pt}c@{\hskip1.5pt}c@{\hskip5.0pt}c@{\hskip0.0pt}}
         & density 1k & samples 1k & density 5k & samples 5k & \\
    \rotatebox{90}{\quad Moser Flow} &
    \includegraphics[width=0.18\columnwidth]{images/bird/MF/cropped/generated_density_1k.png}
         &
         \includegraphics[width=0.18\columnwidth]{images/bird/MF/cropped/generated_samples_1k.png}
         & 
         \includegraphics[width=0.18\columnwidth]{images/bird/MF/cropped/generated_density_5k.png}
         &
         \includegraphics[width=0.18\columnwidth]{images/bird/MF/cropped/generated_samples_5k.png}
         &
         \includegraphics[width=0.25\columnwidth]{images/bird/train_nll_vs_total_time.png} 
         \\
          \rotatebox{90}{\qquad  FFJORD}  &
          \includegraphics[width=0.18\columnwidth]{images/bird/FFJORD/cropped/generated_density_1k.png}
         &
         \includegraphics[width=0.18\columnwidth]{images/bird/FFJORD/cropped/generated_samples_1k.png}
         & 
         \includegraphics[width=0.18\columnwidth]{images/bird/FFJORD/cropped/generated_density_5k.png}
         &
         \includegraphics[width=0.18\columnwidth]{images/bird/FFJORD/cropped/generated_samples_5k.png}
         & 
         \includegraphics[width=0.25\columnwidth]{images/bird/iteration_times.png}
         \\
         \rotatebox{90}{\quad Moser Flow} &
        \includegraphics[width=0.18\columnwidth]{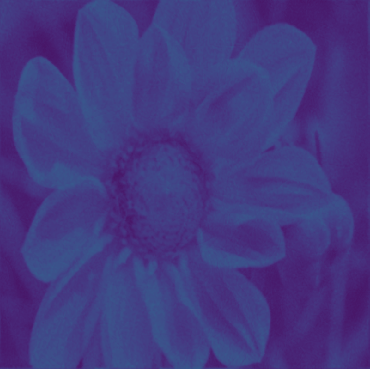}
         &
         \includegraphics[width=0.18\columnwidth]{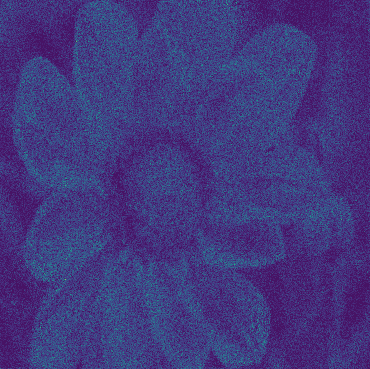}
         & 
         \includegraphics[width=0.18\columnwidth]{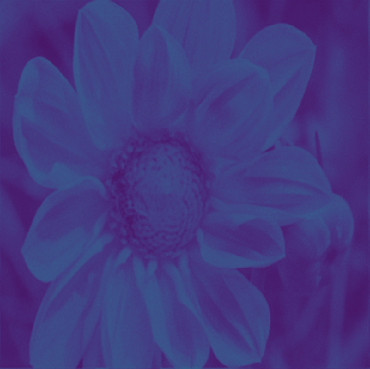}
         &
         \includegraphics[width=0.18\columnwidth]{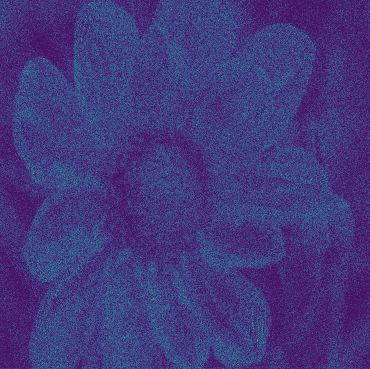}
         &
         \includegraphics[width=0.25\columnwidth]{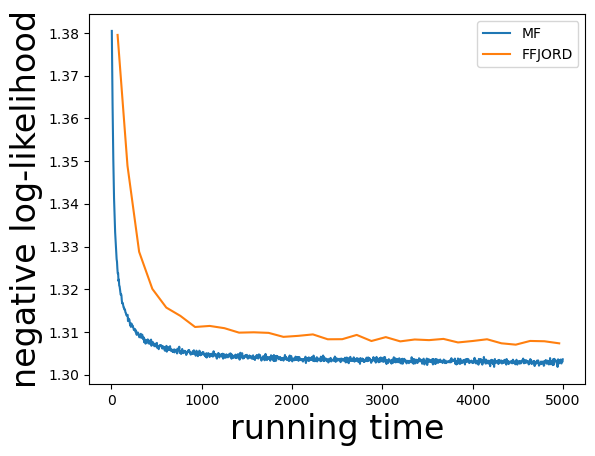} 
         \\
          \rotatebox{90}{\qquad  FFJORD}  &
          \includegraphics[width=0.18\columnwidth]{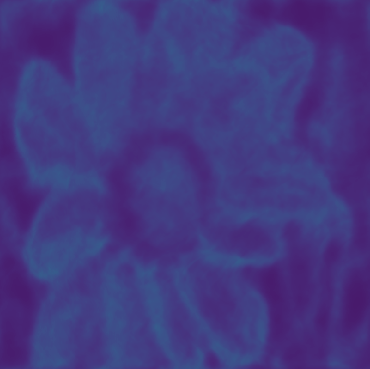}
         &
         \includegraphics[width=0.18\columnwidth]{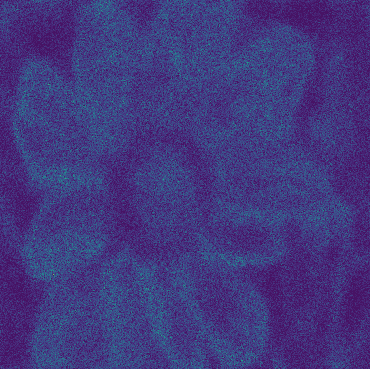}
         & 
         \includegraphics[width=0.18\columnwidth]{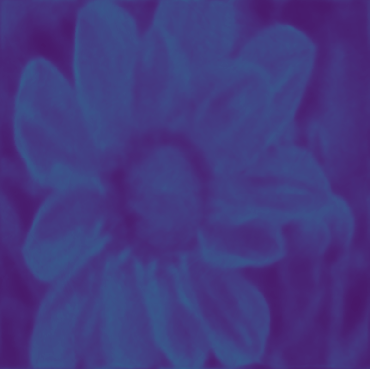}
         &
         \includegraphics[width=0.18\columnwidth]{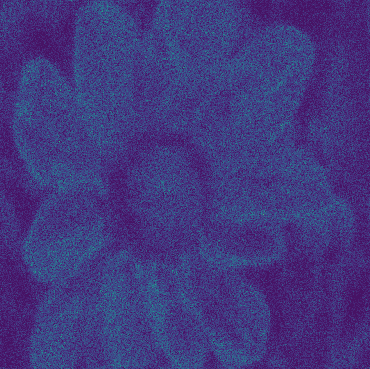}
         & 
         \includegraphics[width=0.25\columnwidth]{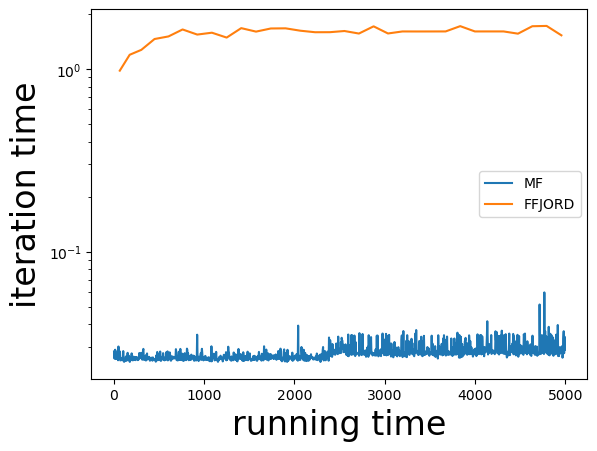}
         \\
    \end{tabular}
    \caption{Comparing learned density and generated samples with MF and FFJORD at different times (in k-sec); top right shows NLL scores for both MF and FFJORD at different times; bottom right shows time per iteration (in $\log$-scale, sec) as a function of total running time (in sec); FFJORD iterations take longer as training progresses. Flickr images (license CC BY 2.0): Bird by Flickr user "lakeworth" \url{https://www.flickr.com/photos/lakeworth/46657879995/}; Flower by Flickr user "daiyaan.db" \url{https://www.flickr.com/photos/daiyaandb/23279986094/}.}
    \label{fig:cameraman}
\end{figure}
\end{document}